%% file: 000MPV.tex
\documentclass[runningheads]{llncs}

\newif\ifarxiv
\arxivtrue 

\input{definitions}

\begin{document}
\mainmatter
\def\ECCVSubNumber{5323}
\title{\PLoneP{} - Point-line Minimal Problems under Partial Visibility in Three Views}

\ifarxiv
\author{Timothy Duff${}^1$, Kathlén Kohn${}^2$, Anton Leykin${}^3$, Tomas Pajdla${}^4$}
\institute{${}^{1,3}$Georgia Tech Atlanta; ${}^2$KTH Stockholm; ${}^4$CIIRC, CTU in Prague}
\titlerunning{ } 
\authorrunning{ } 
\else
\author{Anonymous ECCV submission}
\institute{Paper ID \ECCVSubNumber}
\titlerunning{ECCV-20 submission ID \ECCVSubNumber} 
\authorrunning{ECCV-20 submission ID \ECCVSubNumber} 
\fi

\maketitle
\begin{abstract}
We present a complete classification of minimal problems for generic arrangements of points and lines in space observed partially by three calibrated perspective cameras when each line is incident to at most one point. This is a large class of interesting minimal problems that allows  missing observations in images due to occlusions and missed detections. There is an infinite number of such minimal problems; however, we show that they can be reduced to 140616 equivalence classes by removing superfluous features and relabeling the cameras. We also introduce camera-minimal problems, which are practical for designing minimal solvers, and show how to pick a simplest camera-minimal problem for each minimal problem. This simplification results in 74575 equivalence classes. Only 76 of these were known; the rest are new. In order to identify problems that have potential for practical solving of image matching and 3D reconstruction, we present several smaller natural subfamilies of camera-minimal problems as well as compute solution counts for all camera-minimal problems which have less than 300 solutions for generic data. 
\end{abstract}
\section{Introduction}
\noindent Minimal problems~\cite{Nister-5pt-PAMI-2004,Stewenius-ISPRS-2006,kukelova2008automatic,Byrod-ECCV-2008,DBLP:conf/cvpr/RamalingamS08,Elqursh-CVPR-2011,mirzaei2011optimal,DBLP:conf/eccv/KneipSP12,Hartley-PAMI-2012,kuang-astrom-2espc2-13,Kuang-ICCV-2013,saurer2015minimal,ventura2015efficient,DBLP:conf/eccv/CamposecoSP16,SalaunMM-ECCV-2016,larsson2017efficient,Larsson-Saturated-ICCV-2017,larsson2017making,AgarwalLST17,Barath-CVPR-2017,Barath-CVPR-2018,Barath-TIP-2018,Miraldo-ECCV-2018},\linebreak which we study, are 3D reconstruction problems recovering camera poses and world coordinates from given images such that random input instances have a finite positive number of solutions. They are important basic computational tasks in 3D reconstruction from images
\cite{Snavely-SIGGRAPH-2006,snavely2008modeling,schoenberger2016sfm}, image matching~\cite{rocco2018neighbourhood}, visual odometry and localization \cite{Nister04visualodometry,Alismail-odometry,Sattler-PAMI-2017,taira2018inloc}. Recently, a complete characterization of minimal problems for points, lines and their incidences in calibrated multi-view geometry appeared for the case of complete multi-view visibility~\cite{PLMP}. In this paper, we extend the characterization to an important class of problems under {\em partial} multi-view visibility. 

\noindent {\bf Contribution.} We provide a complete classification of minimal problems for generic arrangements of points and lines in space observed partially by three calibrated perspective cameras when each line is incident to at most one point. 
%
There is an infinite number of such minimal problems; however, we show that they can be {\em reduced} to 140616 equivalence classes of {\em reduced minimal} problems by removing superfluous features and relabeling the cameras. We compute a full description of each class in terms of the incidence structure in 3D and visibility of each 3D feature in images. All problems in every equivalence class have the same {\em algebraic degree}, \ie the number of solutions over the complex numbers. 

When using minimal solvers to find correct image matches by RANSAC \cite{Fischler-Bolles-ACM-1981,Raguram-USAC-PAMI-2013}, we often aim to recover camera parameters only. We name such reconstruction problems {\em camera-minimal} and reserve ``minimal'' for when we aim to recover 3D structure as well. Note that minimal problems are also camera-minimal but not vice versa. For instance, 50 out of the 66 problems given in~\cite{Kileel-MPCTV-2016} are non-minimal yet they all are camera-minimal. As an example, consider the problem from~\cite{Kileel-MPCTV-2016} with 3 PPP and 1 PPL correspondences. It is camera-minimal, \ie there are 272 (in general complex) camera solutions, but it is not minimal since the line of the PPL correspondence cannot be recovered uniquely in 3D: there is a one-dimensional pencil of lines in 3D that project to the observed line in one of the images. 

For each minimal problem, we delete additional superfluous features in images that can be removed without loosing camera-minimality in order to obtain a simplest camera-minimal problem. Thus, we introduce {\em terminal camera-minimal} problems.
We show that,  up to relabeling cameras, there are 74575 of these.
They form the comprehensive list worth studying, since a solver for any camera-minimal problem can be derived from a solver for some problem on this list.  
Only 76 of the 74575 terminal camera-minimal problems were known --- the 66 problems listed in \cite{Kileel-MPCTV-2016} plus 10 additional cases from~\cite{PLMP}  --- the remaining 74499, to the best of our knowledge, are new! We find all terminal camera-minimal problems with less than 300 solutions for generic data and present other interesting cases that might be important for practical solving of image matching and 3D reconstruction.

Characterizing minimal problems under partial visibility, which allows for missing observations in images due to occlusions and missed detections, is very hard. Previous results in~\cite{PLMP} treat the case of full visibility with no restrictions on the number of cameras and types of incidences, resulting in 30 minimal problems. By contrast, we construct a long list of interesting problems under partial visibility, even with our restrictions, \ie having exactly three cameras and having each line incident to at most one point\footnote{Under this restriction, in two cameras, the only reduced (camera-)minimal problem is the five-point problem; see Supplementary Material (SM) for an explanation.}.
These restrictions make the task of enumerating problems tractable while making it still possible to account for very practical incidence cases where several existing feature detectors are applicable. For instance, SIFT~\cite{DBLP:journals/ijcv/Lowe04} and LAF~\cite{DBLP:conf/icpr/MatasOC02} provide quivers (points with one direction attached), which can be interpreted as lines through the points and used to compute relative camera poses~\cite{Fabri-ArXiv-2019}.
\section{Previous work}\label{sec:previous-work}
\noindent A large number of minimal problems appeared in the literature. See references above and~\cite{Larsson-CVPR-2018,kukelova2017clever,Kileel-MPCTV-2016,PLMP} and references therein for work on general minimal problems. Here we review the most relevant work for minimal problems in three views related to point-line incidences and their classification.

Correspondences of non-incident points and lines in three uncalibrated views are considered in early works on the trifocal tensor~\cite{Hartley-IJCV-1997}. Point-line incidences in the uncalibrated setup are introduced in~\cite{Johansson-ICVGIP-2002} as n-quivers (points incident with n lines) and minimal problems for three 1-quivers in three affine views and three 3-quivers in three perspective views are derived. General uncalibrated multi-view constraints for points, lines and their incidences are presented in~\cite{MaHVKS-IJCV-2004}. Non-incident points and lines in three uncalibrated images also appear in~\cite{Oskarsson-IVC-2004}. The cases of four points and three lines, two points and six lines, and nine lines are studied; \cite{Larsson-Syzygy-CVPR-2017} constructs a solver for nine lines. Work \cite{Fabbri-IJCV-2016} looks at lines incident to points which arise from tangent lines to curves and~\cite{Fabri-ArXiv-2019} presents a solver for that case. Results~\cite{JoswigKSW16,AgarwalLST17,AholtO14,Aholt-1107-2875,Trager-PhD-2018,Ponce-IJCV-2016,PLMP} introduced some of the techniques that are useful for classifying  classes of minimal problems.

{\bf The most relevant previous work} is~\cite{Kileel-MPCTV-2016} and~\cite{PLMP}. Work~\cite{Kileel-MPCTV-2016} classifies camera-minimal problems in three calibrated views that can be formulated using linear constraints on the trifocal tensor~\cite{HZ-2003}. It presents 66 camera-minimal problems. These are all covered in our classification and are all terminal camera-minimal. Among them there are 16 reduced minimal problems out of which 2 are with full visibility and 14 are with partial visibility. The remaining 50 problems are not minimal.  

A complete characterization of minimal problems for points, lines and their incidences in calibrated multi-view geometry for the case of complete multi-view visibility is presented in~\cite{PLMP}. It gives 30 minimal problems. Among them, 17 problems include exactly three cameras but only 12 of them ($3002_1$, $3002_2$, $3010_0$, $2005_3$, $2005_4$, $2005_5$, $2013_2$, $2013_3$, $2021_1$, $1024_4$, $1032_2$, $1040_0$ in Tab.~1 of~\cite{PLMP}) meet our restrictions on incidences. These 12 cases are all terminal camera-minimal as well as reduced minimal. Notice that the remaining 5 problems  ($3100_0$, $2103_1$, $2103_2$, $2103_3$, $2111_1$ in Tab.~1 of~\cite{PLMP}) are not considered in this paper because they include three collinear points in 3D and collinearity of more than two points cannot be modeled without including a line going through more than one point.

This paper can be seen as an extension of~\cite{Kileel-MPCTV-2016} and~\cite{PLMP} to a much larger class of problems in three calibrated views under partial multi-view visibility. 

\section{Problem Specification}
Our main result applies to problems in which points, lines, and point-line incidences are partially observed.
We will model intersecting lines by requiring that each intersection point of two lines has to be one of the points in the point-line problem.
\begin{definition} \label{def:PLP}
\rm
A \emph{point-line problem} is a tuple $(\PLP)$ specifying that $p$ points and $l$ lines in space satisfy a given incidence relation $$\mathcal{I}\subset \{ 1, \ldots , p \} \times \{ 1, \ldots , l \},$$ 
where $(i,j)\in \mathcal{I}$ means that the $i$-th point is on the $j$-th line, and are projected to $m=|\obs|$ views with
$$\obs = ((\mathcal{P}_1, \mathcal{L}_1), \ldots, (\mathcal{P}_m, \mathcal{L}_m))$$ describing which points and lines are observed by each camera ---
view $v$ contains exactly the points in $\mathcal{P}_v \subset \lbrace 1, \ldots, p \rbrace$ 
and the lines in $\mathcal{L}_v \subset \lbrace 1, \ldots, l \rbrace$.

For  $\mathcal{I}$ we assume \emph{realizability} (the incidence relations are realizable by some point-line arrangement in $\RR^3$) and \emph{completeness} (every incidence which is automatically implied by the incidences in $\cI$ must also be contained in $\cI$).

For  $\mathcal{O}$ we assume that if a camera observes two lines that meet according to $\mathcal{I}$ then it observes their point of intersection.
\end{definition}
Note that, for instance, the realizability assumption implies that two distinct lines cannot have more than one point in common.
Our assumption on $\mathcal{O}$ is natural ---
the set $\mathcal{I}$ of incidences describes all the knowledge about which lines intersect in space, as well as in the images.
An \emph{instance} of a point-line problem is specified by the following data:

(1) A point-line arrangement in space consisting of $p$ points $X_1,\ldots,X_p$ and $l$ lines $L_1,\ldots,L_l$ in $\PP^3$ which are incident exactly as specified by $\mathcal{I}$.
Hence, the point $X_i$ is on the line $L_j$ if and only if $(i,j)\in \mathcal{I}$. We write 
\[
\PplI = \left\{ (X, L) \in \left( \PP^3 \right)^p  \times  \left( \GG_{1,3} \right)^{l} \mid \forall (i,j) \in \mathcal{I}\, \colon X_i \in L_j   \right\}
\]
for the associated \emph{variety of point-line arrangements}.
Note that this variety also contains degenerate arrangements,
where not all points and lines have to be pairwise distinct
or where there are more incidences between points and lines than those specified by $\mathcal{I}$.

(2) A list of $m$ calibrated cameras which are represented by matrices \[
P_1 = [R_1 \mid t_1], \ldots, P_m = [R_m \mid t_m]
\]
with $R_1, \ldots, R_m \in \SO (3)$ and $t_1, \ldots, t_m \in \RR^{3}$. 

(3) The \emph{joint image} consisting of the projections $ \lbrace x_{v,i} \mid i \in \mathcal{P}_v \rbrace  \subset \PP^2$ of the points $X_1, \ldots, X_p$ and the projections $ \lbrace \ell_{v,j} \mid j \in \mathcal{L}_v \rbrace \subset \GG_{1,2}$ of the lines $L_1, \ldots,L_l$ by the cameras $P_1,\ldots,P_m$
to the  views $v = 1,\ldots,m$. 
We denote by $\rho = \sum_{v=1}^m |\mathcal{P}_v|$
and $\lambda = \sum_{v=1}^m |\mathcal{L}_v|$
the total numbers of observed points and lines, and write
$$
\YplIO = \left\{ (x, \ell) \in \left(\PP^2 \right)^{\rho} \times \left( \GG_{1,2} \right)^{\lambda} \;\middle\vert\;
\begin{array}l
      \forall v=1,\ldots,m \;   
\forall i \in \mathcal{P}_v \;
\forall j \in \mathcal{L}_v: \\
 (i,j) \in \mathcal{I}
\Rightarrow
x_{v, i} \in \ell_{v, j}
\end{array}
\right\} 
$$
for the \emph{image variety}
which consists of all $m$-tuples of 2D-arrangements of the points and lines specified by $\obs$ which satisfy the incidences specified by~$\cI$.
We note that an $m$-tuple in $\YplIO$ is not necessarily a joint image of a common point-line arrangement in $\PP^3$.

Given a joint image, we want to recover an arrangement in space and cameras yielding the given joint image. We refer to a pair of such an arrangement and such a list of $m$ cameras as a \emph{solution} of the point-line problem for the given joint image.

To fix the arbitrary space coordinate system~\cite{HZ-2003}, we set $P_1 = [I\,|\,0]$ and the first coordinate of $t_2$ to $1$. So our \emph{camera configurations} are parameterized by
$$
\cams{m} = \left\{ (P_1, \ldots, P_m) \in \left(\RR^{3\times 4}\right)^{m} \;\middle\vert\; \begin{array}{l}
     P_i = [R_i \mid t_i], \,
R_i \in \SO (3), \,  t_i \in \RR^3,\\ R_1 = I, \, t_1 = 0, \, t_{2,1}=1
\end{array}
\right\}.
$$
We will always assume that the camera positions in an instance of a point-line problem are sufficiently generic such that the 
points and lines in the views are in generic positions with respect to the specified incidences $\mathcal{I}$. 

We say that a point-line problem is \emph{minimal} if a generic image tuple in $\YplIO$ has a nonzero finite number of solutions.
We may phrase this formally:
\begin{definition}
\label{def:minimal}
\rm
Let $\Phi_{\PLP}: \PplI \times \cams{m} \dashrightarrow \YplIO$
denote the \emph{joint camera map}, which sends a point-line arrangement in space and $m$ cameras to the resulting joint image. 
We say that the point-line problem $(\PLP)$ is \emph{minimal}~if

\noindent $\bullet\  \Phi_{\PLP}$ is a \emph{dominant map}\footnote{In birational geometry, dominant maps are analogs of surjective maps.}, \ie a generic element $(x, \ell)$ in $\YplIO$ has a solution, so $\Phi_{\PLP}^{-1} (x, \ell) \neq \emptyset$, and

\noindent $\bullet\ $the preimage $\Phi_{\PLP}^{-1} (x, \ell)$ of a generic element $(x, \ell)$ in $\YplIO$ is finite.
\end{definition}

\begin{remark}
\label{remark:perturbation}
We require the joint camera map in Definition~\ref{def:minimal} to be dominant because we want solutions to minimal problems to be stable under perturbation of the image data that preserves the incidences $\mathcal{I}$. A classical example of a problem which is not stable under perturbation in images is the problem of four points in three calibrated views~\cite{Nister-IJCV-2006}.
\end{remark}

\noindent
Over the complex numbers, the cardinality of the preimage $\Phi_{\PLP}^{-1} (x, \ell)$ is the same for every \emph{generic} joint image $(x, \ell)$ of a minimal point-line problem $(\PLP)$.
We refer to this cardinality as the \emph{degree} of the minimal problem.

In many applications, one is only interested in recovering the camera poses, and not the points and lines in 3D.
Hence, we say that a point-line problem is \emph{camera-minimal} if, given a generic image tuple in $\YplIO$, it has a nonzero finite number of possible camera poses. Formally, this means:

\begin{definition} \rm
\label{def:camera-minimal}
Let $\gamma: \PplI  \times \cams{m} \to \cams{m}$ denote the projection onto the second factor.
We say that the point-line problem $(\PLP)$ is \emph{camera-minimal} if

\noindent $\bullet\ $ its joint camera map $\Phi_{\PLP}$ is dominant, and

\noindent $\bullet\ \gamma(\Phi_{\PLP}^{-1}(x,\ell))$ is finite for a generic element $(x, \ell)$ in $\YplIO$.

\noindent
The cardinality over $\CC$ of a generic $\gamma(\Phi_{\PLP}^{-1}(x,\ell))$ is the {\em camera-degree} of $(\PLP)$.
\end{definition}

\begin{remark}
Every minimal point-line problem is camera-minimal, but not necessarily the other way around. 
In the setting of complete visibility (i.e. where every camera observes all points and all lines), both notions coincide~\cite[Cor.~2]{PLMP}.
\end{remark}  
In~\cite{PLMP}, all minimal point-line problems with complete visibility are described, including their degrees. 
It is a natural question if one can extend the classification in~\cite{PLMP} to all point-line problems with \emph{partial visibility}.
A first obstruction is that there are minimal point-line problems for arbitrarily many cameras\footnote{See SM discussion of camera registration.},
whereas the result in~\cite{PLMP} shows that minimal point-line problems with complete visibility exist only for at most six views.
Moreover, as we see in the following sections, deriving a classification for partial visibility seems more difficult and involves more elaborate tools.
Hence, in this article, we only aim for classifying point-line problems \emph{in three views}\footnote{See SM for a discussion on two views.}.
We also restrict our attention to point-line problems satisfying the following assumption:
\begin{definition} \label{def:PL1P} \rm 
We say that a point-line problem is a \emph{\PLkP{1}} if 
each line in 3D is incident to at most one point.
\end{definition}

\noindent This assumption makes our analysis easier, since the point-line arrangement in space of a {\PLoneP} is a collection of the following independent \emph{local features}: 

\noindent $\bullet\ $ \emph{free line} (\ie a line which is not incident to any point), and

\noindent $\bullet\ $ \emph{point with $k$ pins} where $k = 0, 1, 2, \ldots$ (\ie a point with $k$ incident lines).

\noindent $\quad\ $In the following, we shortly write \emph{pin} for a line passing through a point.
We stress that 
a pin refers only to the line itself, 
rather than the incident point.
A first consequence of restricting our attention to \PLoneP{}s is the following fact, which fails for general point-line problems\footnote{See SM for an example.}.

\begin{lemma}
\label{lem:degs-equal}
The degree and camera-degree of a minimal \PLoneP{} coincide.
\end{lemma}
\begin{proof}
Proofs of all lemmas, theorems and justification of results are in SM.
\end{proof}

\noindent
We will see that there are \emph{infinitely many} (camera-)minimal \PLoneP s in three views. However, we can partition them into finitely many classes such that all \PLoneP s in the same class are closely related; in particular, they have the same (camera-)degree.
For this classification, we pursue the following strategy:

\noindent \textbf{Step 1:} We introduce \emph{reduced} \PLoneP s as the canonical representatives of the finitely many classes of minimal \PLoneP{}s we aim to find (see Section~\ref{sec:reducedPL1Ps}).

\noindent \textbf{Step 2:} Basic principles from algebraic geometry brought up in~\cite{PLMP} imply
\begin{lemma}
\label{lem:balancedPlusDominant}
A point-line problem $(\PLP)$ is minimal if and only if

\noindent $\bullet\ $ it is \emph{balanced}, \ie $\dim (\PplI \times \cams{m}) =  \dim ( \YplIO)$, and 

\noindent $\bullet\ $ its joint camera map $\Phi_{\PLP}$ is dominant.
\end{lemma}
 We identify a \emph{finite} list
of reduced balanced \PLoneP s in three views
that contains all reduced minimal \PLoneP s (see Section~\ref{sec:balancedPL1Ps}).

\noindent \textbf{Step 3:} 
We explicitly describe the relation of reduced camera-minimal problems to reduced minimal ones, which implies that there are only finitely many reduced camera-minimal \PLoneP{}s in three views (see Section~\ref{sec:cameraminimal}).

\noindent \textbf{Step 4:} For each
 of the finitely many balanced \PLoneP s identified in Step~2, we check if its joint camera map is dominant.
This provides us with a complete catalog of all reduced (camera-)minimal \PLoneP s in three views
(see Section~\ref{sec:minimality}).

\smallskip
\noindent
In addition to the classification, we compute the camera-degrees of all reduced camera-minimal \PLoneP s in three views whose camera-degree is less than $300$ (see Section~\ref{sec:degrees} for this and related results on natural subfamilies of \PLoneP{}s.)

\section{Reduced \PLoneP{}s}
\label{sec:reducedPL1Ps}
From a given {\PLoneP} $(\PLP)$ we can obtain a new {\PLoneP} by forgetting some points and lines, both in space and in the views. Formally, if $\mathcal{P}' \subset \lbrace 1, \ldots, p \rbrace$ and $\mathcal{L}' \subset \lbrace 1, \ldots , l \rbrace$ are the sets of points and lines which are \emph{not} forgotten, the new {\PLoneP} is $(p', l', \cI', \obs')$ with
$p' = |\mathcal{P}'|$,
$l' = |\mathcal{L}'|$,
$\cI' = \lbrace (i,j)  \in \cI \mid i \in \mathcal{P}', j \in \mathcal{L}' \rbrace$,
and $\obs' = ((\mathcal{P}'_1, \mathcal{L}'_1), \ldots, (\mathcal{P}'_m, \mathcal{L}'_m))$,
where $\mathcal{P}'_v = \mathcal{P}_v \cap \mathcal{P}'$
and $\mathcal{L}'_v = \mathcal{L}_v \cap \mathcal{L}'$.
This induces natural projections $\ppi$
 and $\pi$ between the domains and codomains of the joint camera maps which forget the points and lines \emph{not} in $\mathcal{P}'$ and $\mathcal{L}'$.
 
\begin{center}
        \begin{tikzcd}[framed]
        \PplI \times \cams{m}
        \arrow[r, dashed, "\Phi~=~\Phi_{\PLP}"]
        \arrow[d, "\ppi" left]
         &[20mm] \YplIO  
        \arrow[d, "{\pi}"]
         \\
        \cX_{p',l',\cI'} \times \cams{m}
        \arrow[r, dashed, "\Phi'~=~\Phi_{p',l',\cI',\obs'}" below]
        &[20mm]
        \mathcal{Y}_{p',l',\cI',\obs'}
        \end{tikzcd}
\end{center}

\noindent
In the following, we shortly write $\Phi = \Phi_{\PLP}$
and $\Phi' = \Phi_{p',l',\cI',\obs'}$.

\begin{definition}\label{def:reduced-PL1P}
\rm

\noindent We say that $(\PLP)$ is \emph{reducible} to $(p',l',\cI',\obs')$ if 

\noindent $\bullet\ $
for each forgotten point, at most one of its pins is kept,
and

\noindent $\bullet\ $
a generic solution $S' = ((X',L'),P) \in \cX_{p',l',\cI'} \times \cams{m}$ of $(p',l',\cI',\obs')$ 
can be \linebreak[4]
\indent \hspace{-2mm}
lifted to a solution 
of $(\PLP)$
for generic input images in 
$\pi^{-1}(\Phi'(S'))$.

\hspace{-2mm}
In other words, for a generic
$S' = ((X',L'),P) \in \cX_{p',l',\cI'} \times \cams{m}$
and a generic 

\hspace{-2mm}
$(x, \ell) \in \pi^{-1}(\Phi'(S'))$,
there is a point-line arrangement $(X,L) \in \PplI$ 
such 
\linebreak[4] \indent \hspace{-2mm}
that $\Phi((X,L),P) = (x, \ell)$
and $\ppi((X,L),P) = S'$.
\end{definition}

\begin{theorem}
\label{thm:minimal-reduction}
If a \PLoneP{} is minimal and reducible to another \PLoneP{}, then both are minimal and have the same degree.
\end{theorem}

\noindent
We can partition \emph{all} (infinitely many) minimal \PLoneP s in three views into finitely many classes using this reduction process. 
Each class is represented by a unique {\PLoneP} that is \emph{reduced}, 
\ie not reducible to another \PLoneP .

\begin{theorem}
\label{thm:uniqueReduced}
A minimal {\PLoneP} $(\PLP)$ in three views is reducible to a unique reduced {\PLoneP} $(p',l',\cI',\obs')$.
The corresponding projection $\ppi$ forgets:

\noindent $\bullet\ $
every pin that is observed in exactly two views such that both views also observe

\hspace{-3mm}
the point of the pin 
(it does not matter if the third view observes the point or 

\hspace{-3mm}
not, but it must not see the line), 
e.g.\ \includegraphics[width=0.09\textwidth]{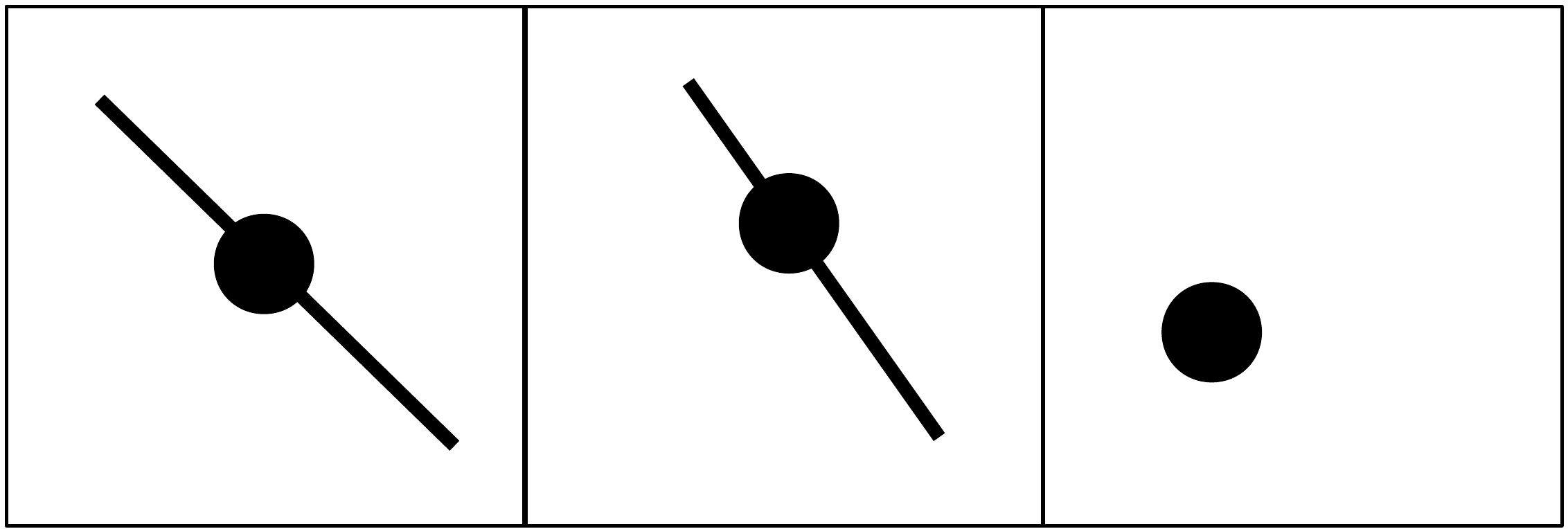} is reduced to \includegraphics[width=0.09\textwidth]{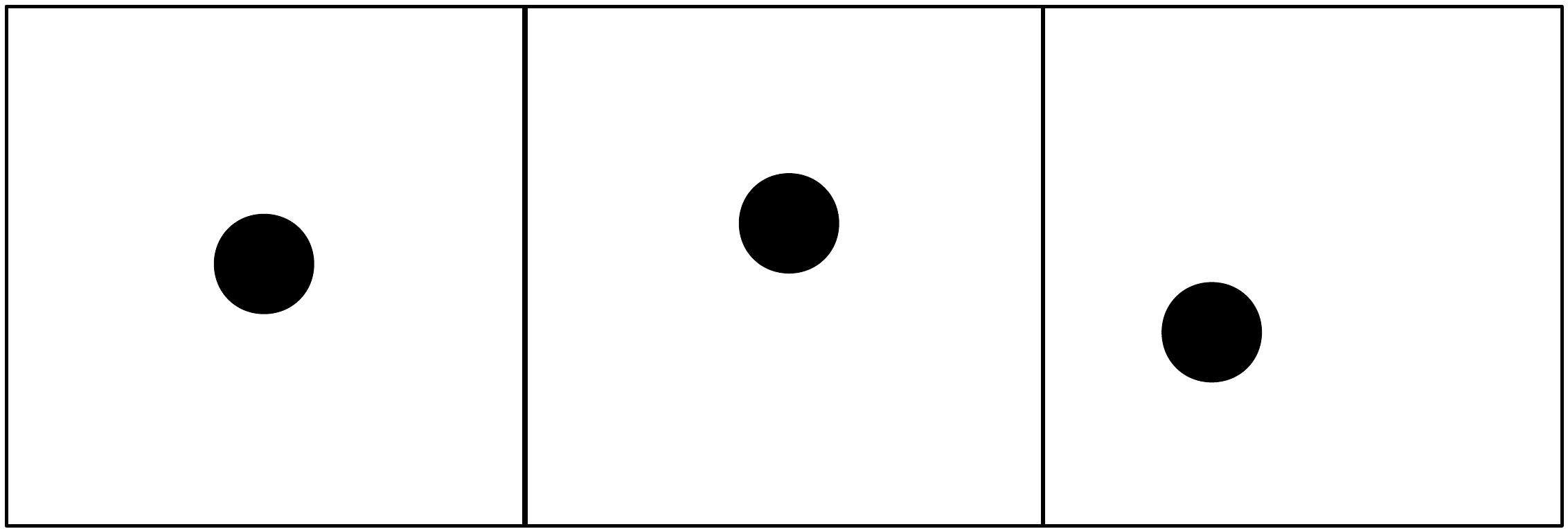}

\noindent $\bullet\ $
every free line that is observed in exactly two views to reduce
\includegraphics[width=0.09\textwidth]{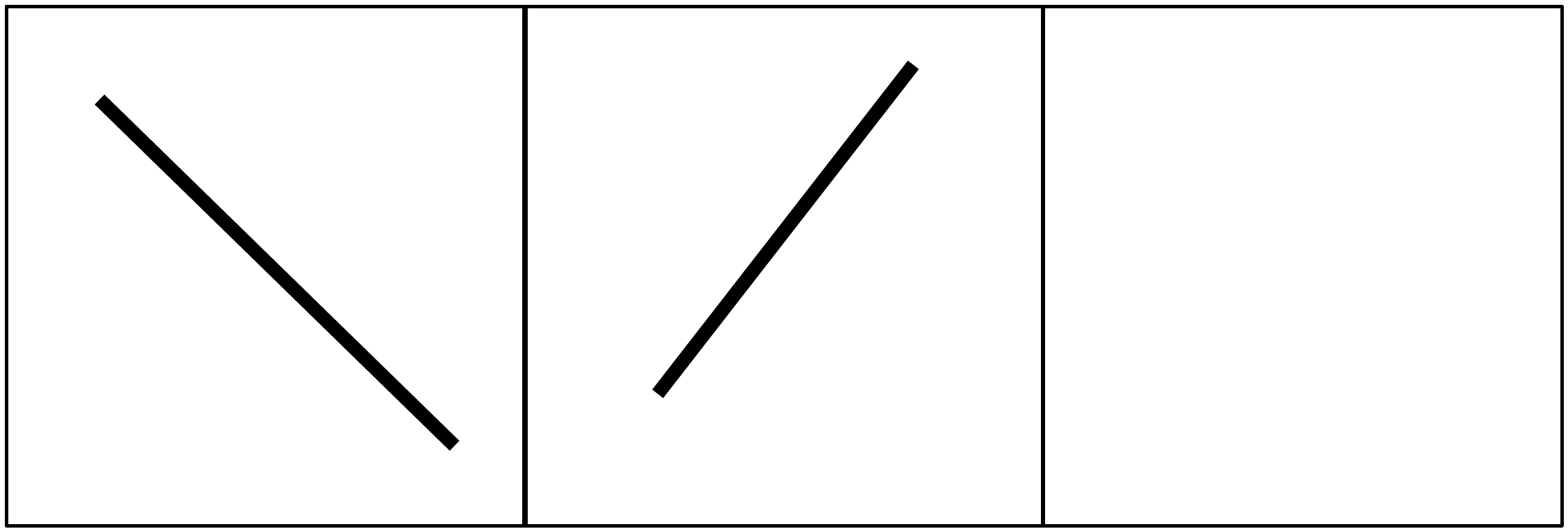} to \includegraphics[width=0.09\textwidth]{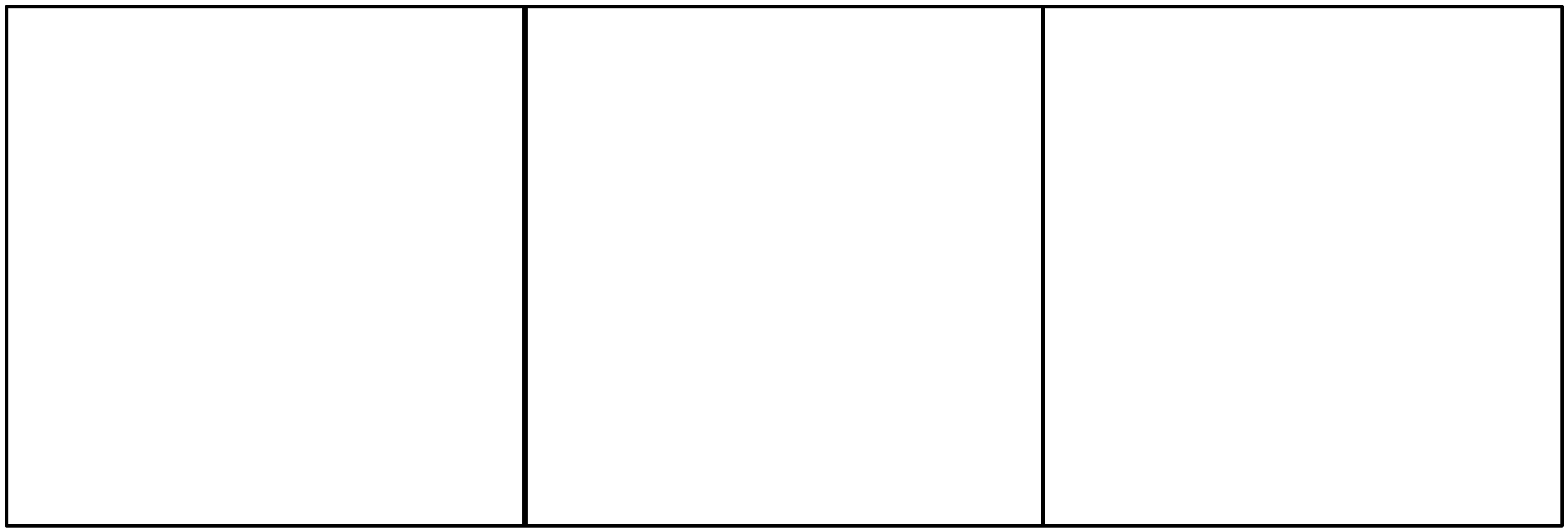} 

\noindent $\bullet\ $
every point that has exactly one pin and is viewed like \includegraphics[width=0.09\textwidth]{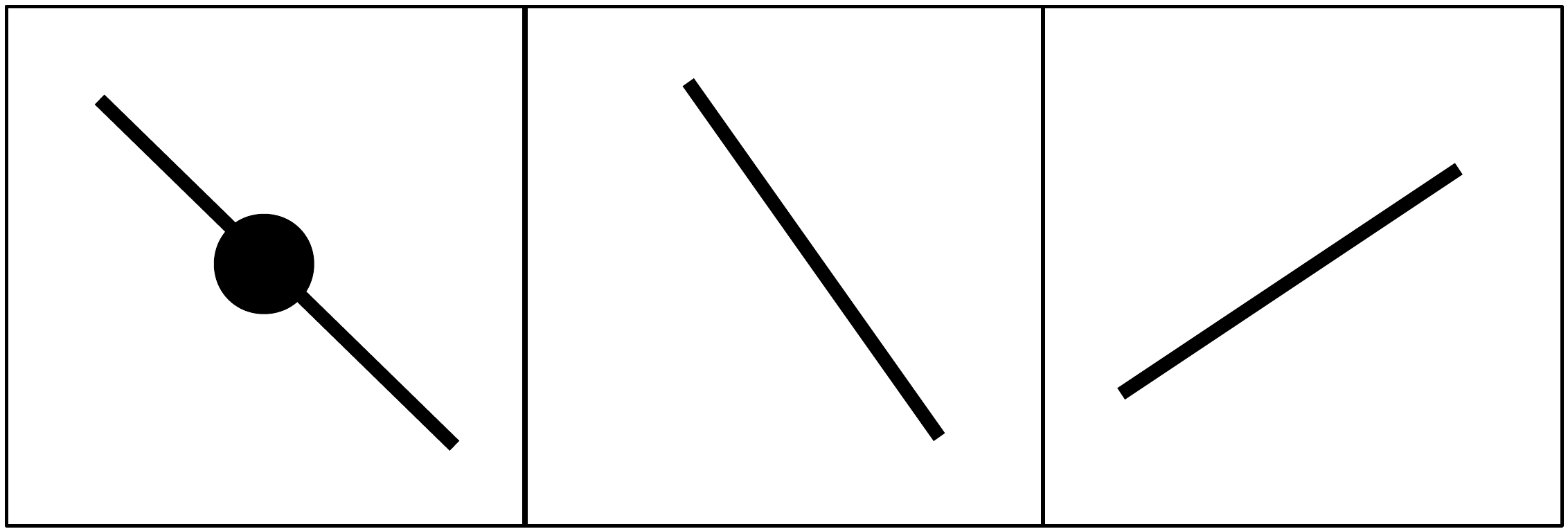}
to get
\includegraphics[width=0.09\textwidth]{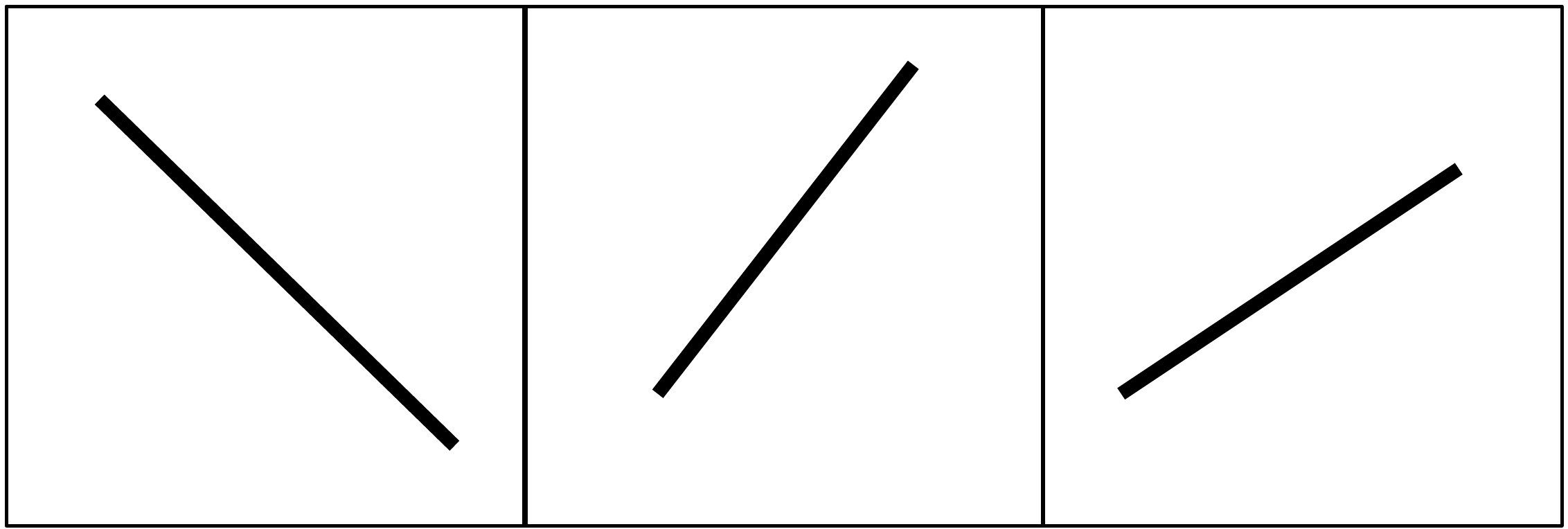}

\noindent $\bullet\ $
every point together with its single pin
if it is viewed like \includegraphics[width=0.09\textwidth]{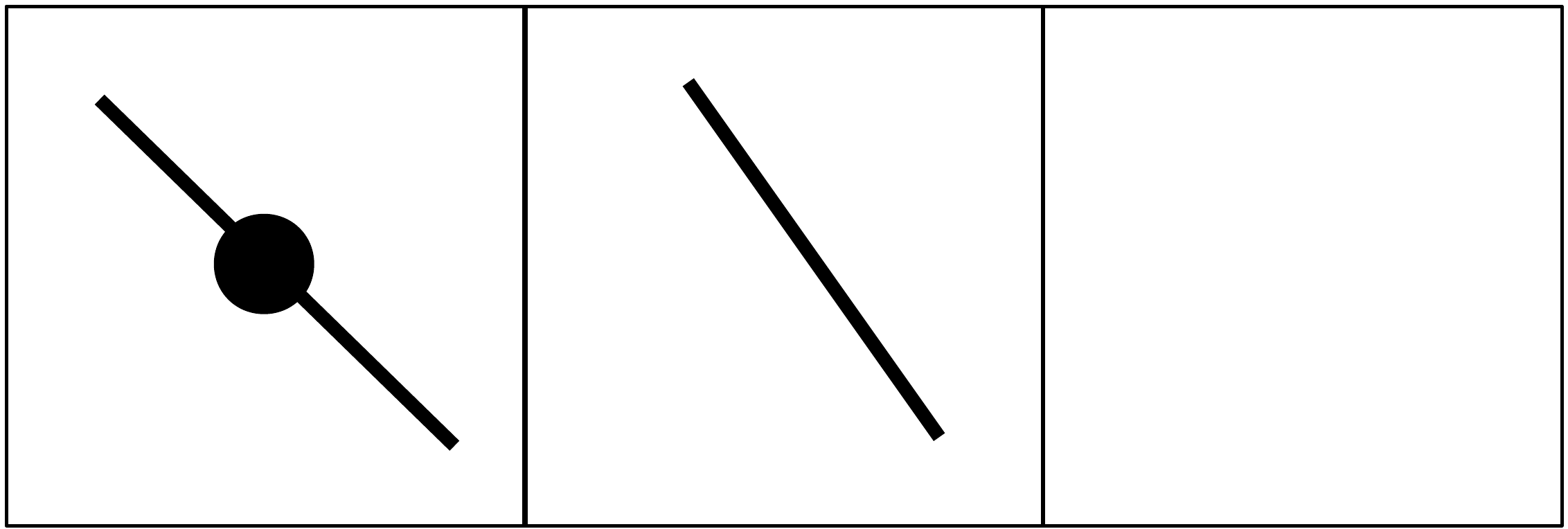}
to get 
\includegraphics[width=0.09\textwidth]{pix/empty.pdf}

\noindent
In addition, applying inverses of these reductions to a minimal \PLoneP{} in three views results in a  minimal \PLoneP{}.

\end{theorem}

\noindent
Hence, it is enough to classify all reduced minimal \PLoneP s. We will see that there are finitely many reduced minimal \PLoneP s in three views. 
To count them, we need to understand how they look.

\begin{theorem}
\label{thm:reducedMinimalLooks}
A reduced minimal {\PLoneP} in three views has at most one point with three or more pins.
If such a point exists,

\noindent $\bullet\ $
it has at most seven pins,

\noindent $\bullet\ $
and the point and all its pins are observed in all three views. 

\noindent
All other local features are viewed as in Table~\ref{tab:localFeatures}.
\end{theorem}

\begin{table}[t]
    \centering
    \begin{tabular}{cccccccccc}
         &\includegraphics[angle=90,origin=c,width=0.07\textwidth]{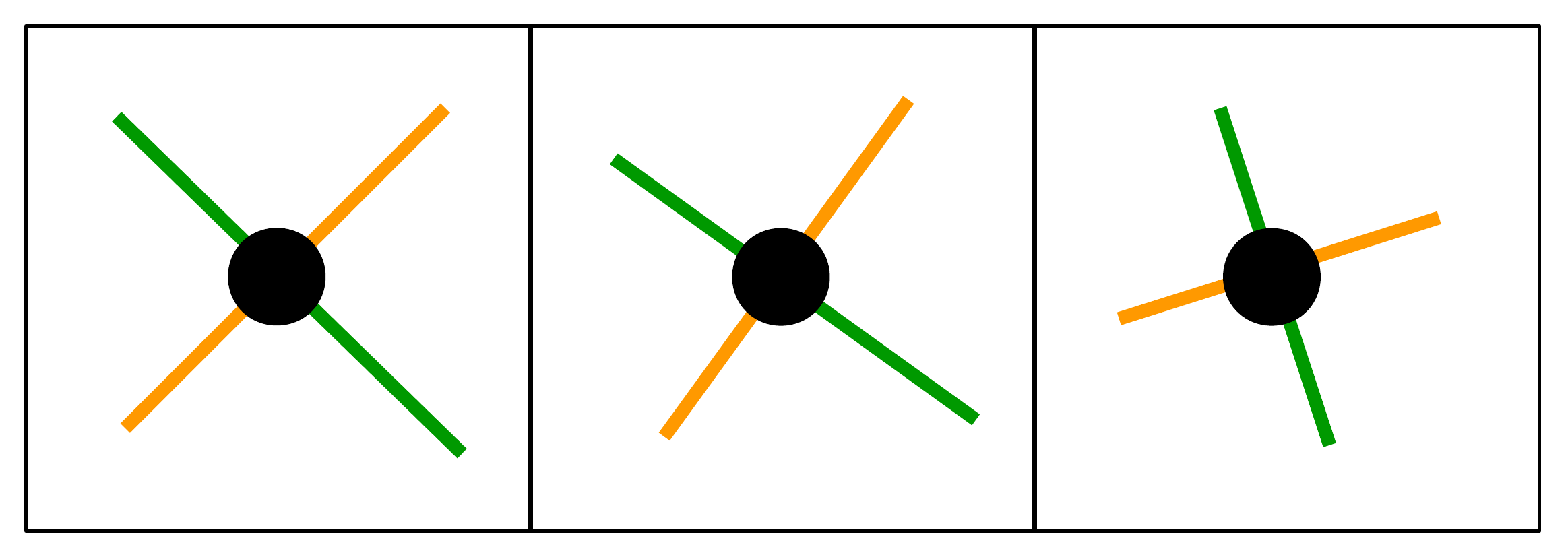}
         &\includegraphics[angle=90,origin=c,width=0.07\textwidth]{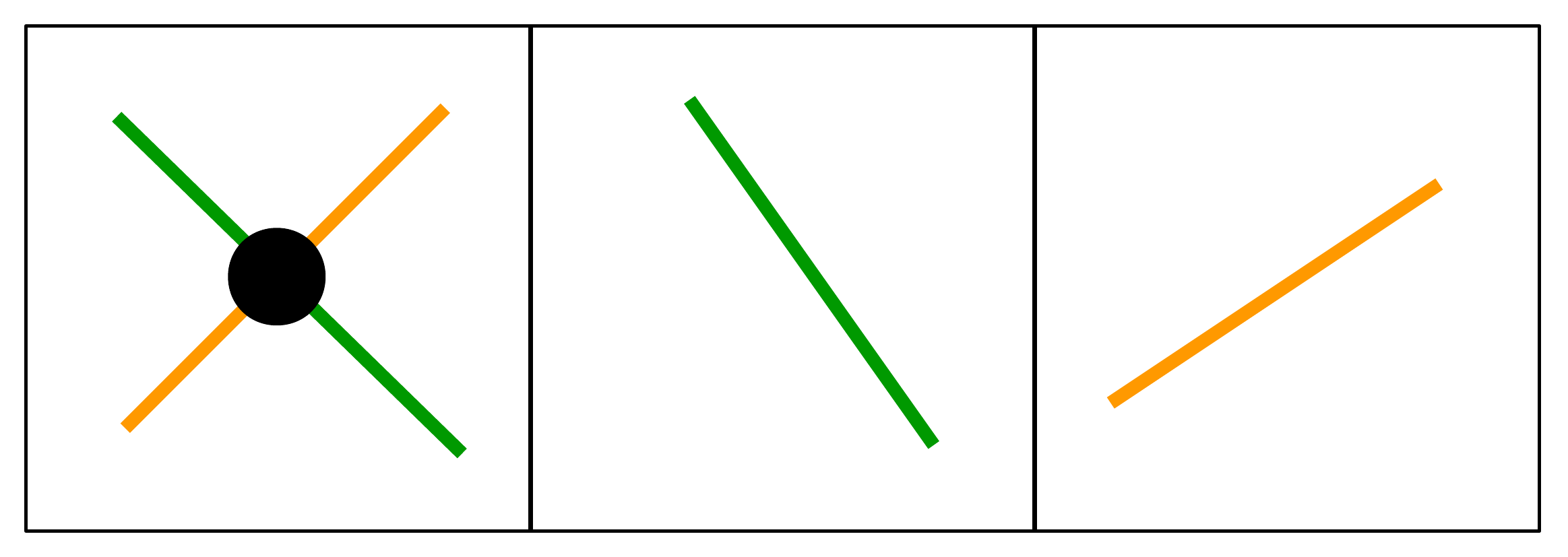}
         &\includegraphics[angle=90,origin=c,width=0.07\textwidth]{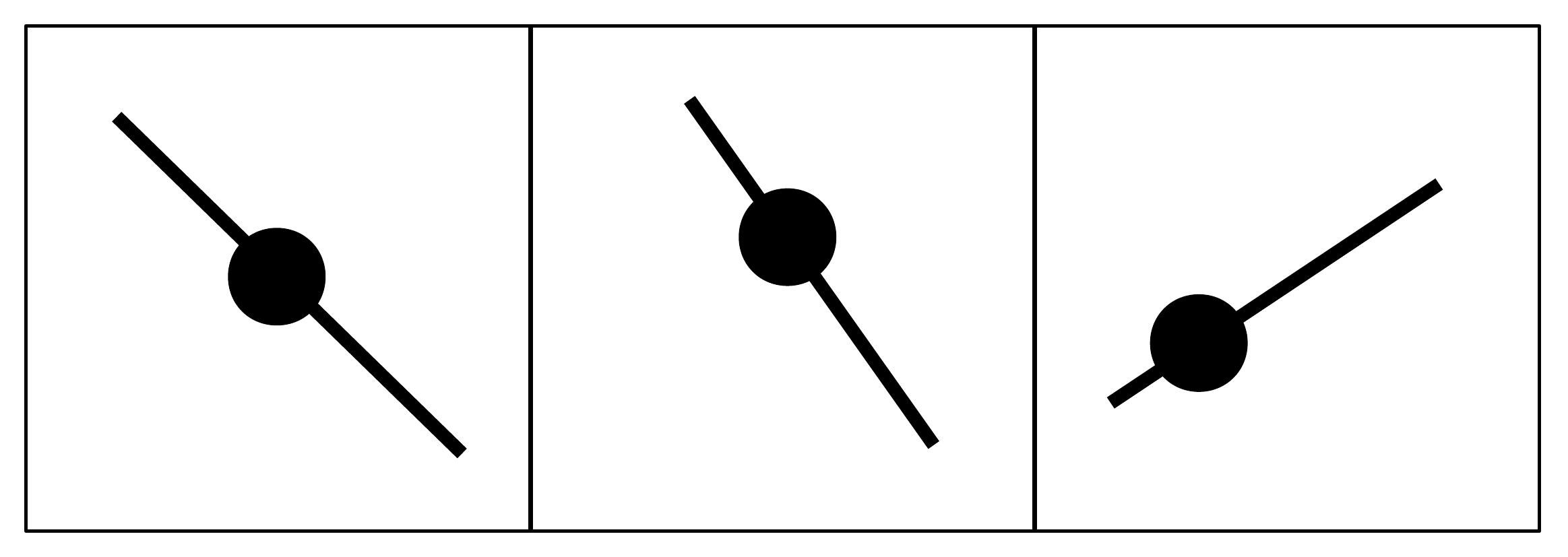}
         &\includegraphics[angle=90,origin=c,width=0.07\textwidth]{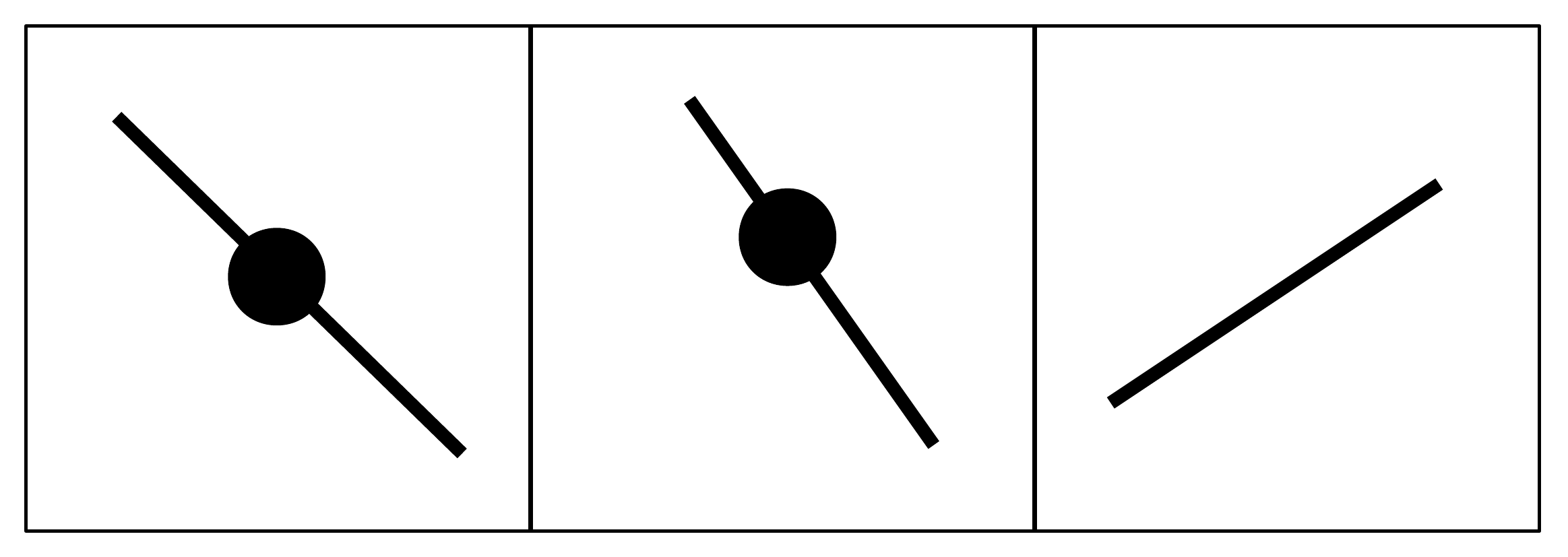}
         &\includegraphics[angle=90,origin=c,width=0.07\textwidth]{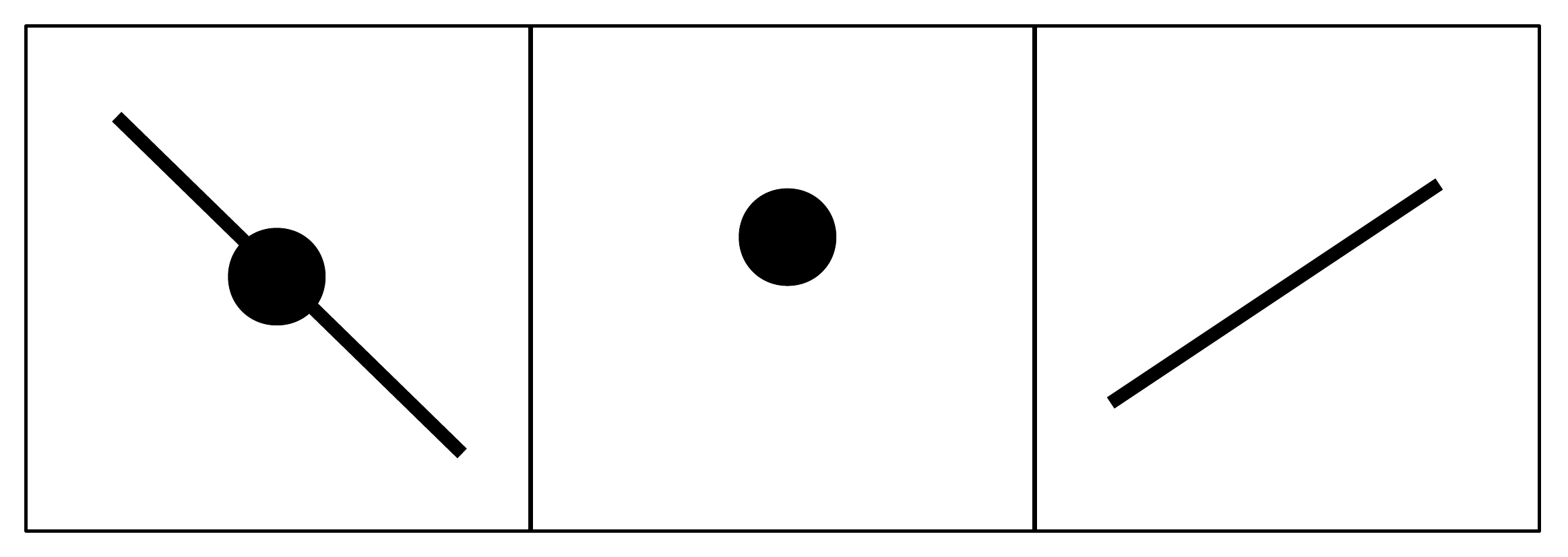}
         &\includegraphics[angle=90,origin=c,width=0.07\textwidth]{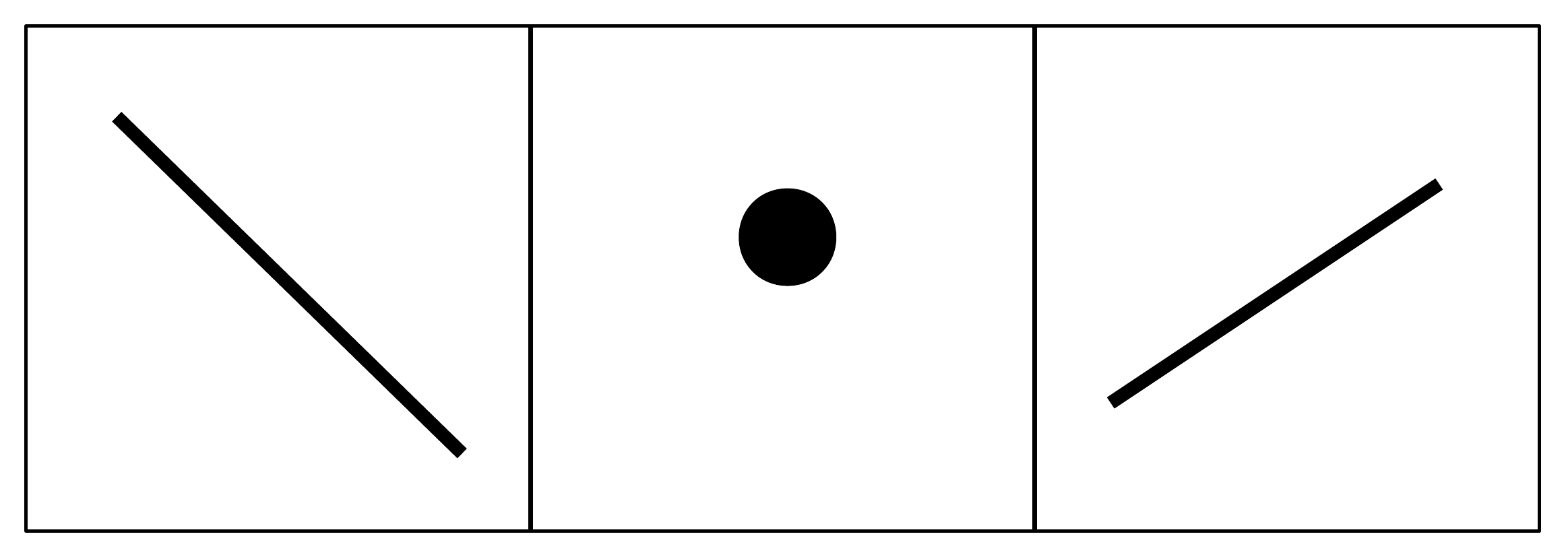}
         &\includegraphics[angle=90,origin=c,width=0.07\textwidth]{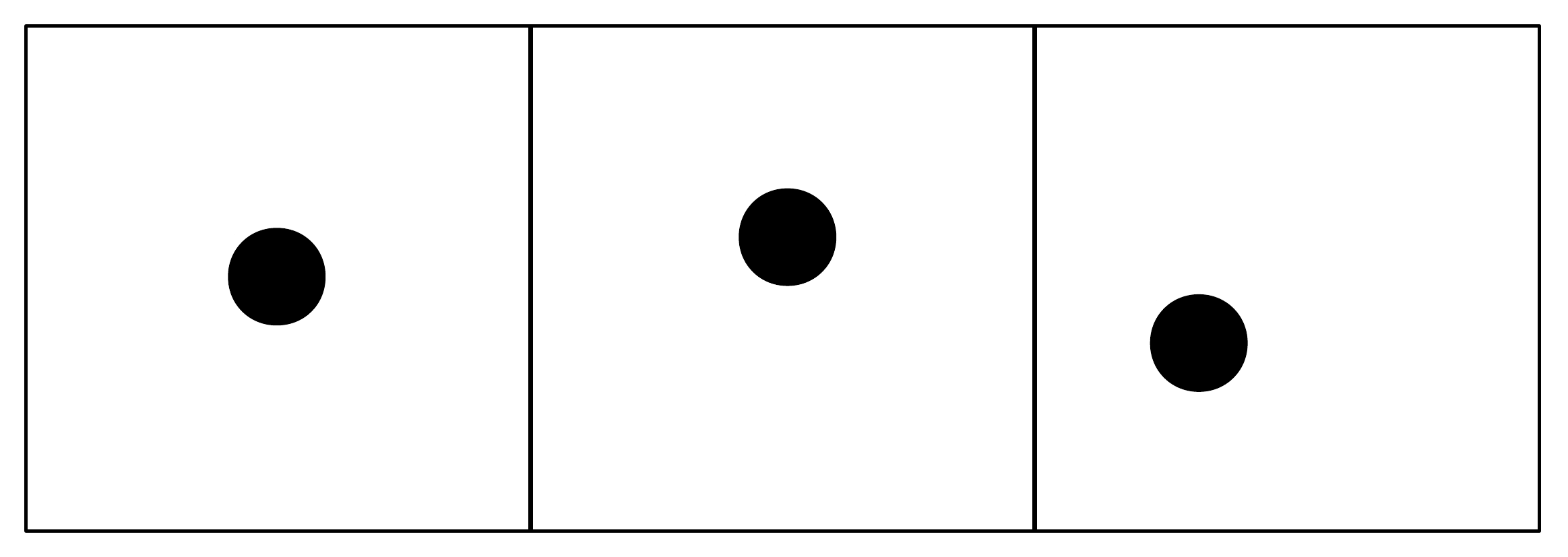}
         &\includegraphics[angle=90,origin=c,width=0.07\textwidth]{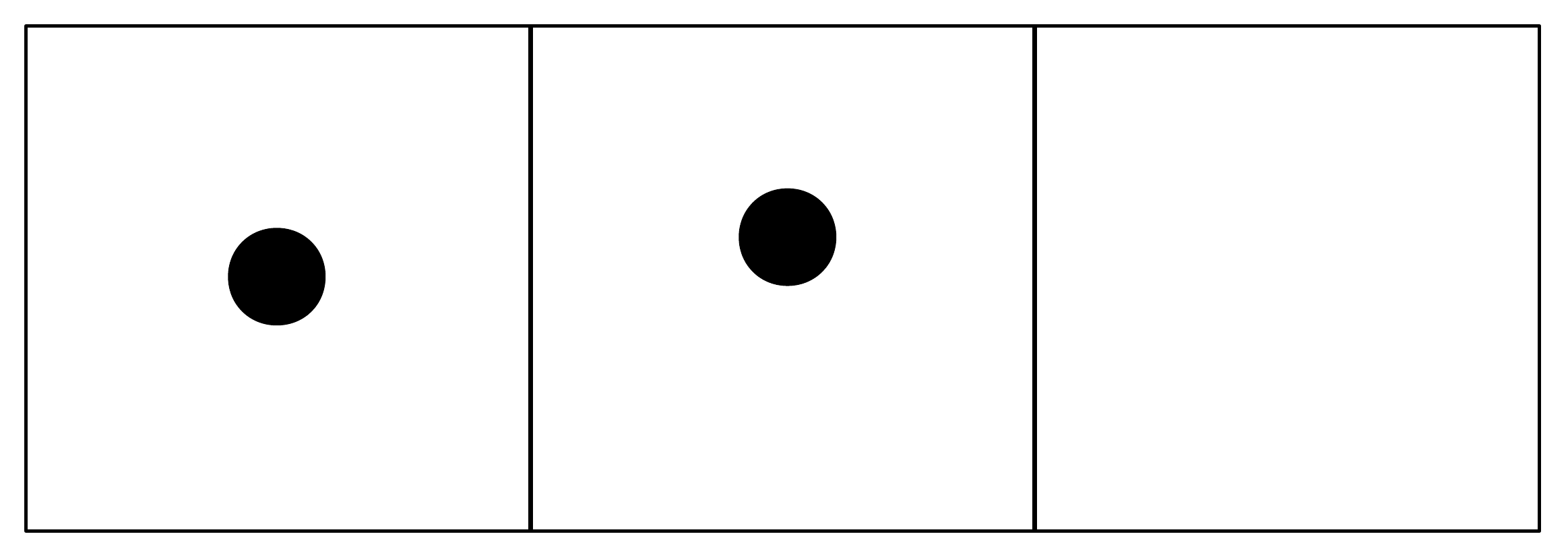}
         &\includegraphics[angle=90,origin=c,width=0.07\textwidth]{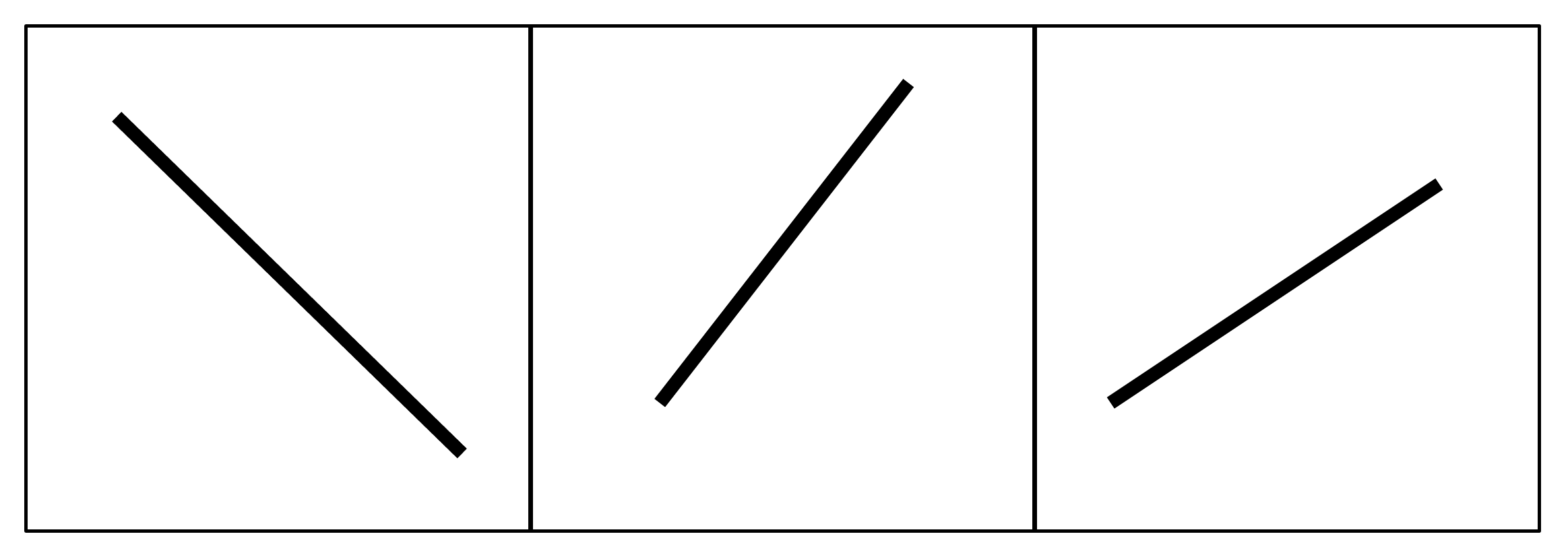}
         \\
         &
        &
        $\ctwoparA$ &
        &
        $\coneparoneA$ &
        $\conepartwoA\vspace{-2mm}$ &
        $\coneparthrA$ &
        &
        $\czerparA$ &
        \\
        \# &
        $\ctwofull$ &
        $\ctwoparB$ &
        $\conefull$ &
        $\coneparoneB$ &
        $\vdots$ &
        $\coneparthrB$ &
        $\czerfull$ &
        $\czerparB$ &
        $\cfree$
        \\
        &
        &
        $\ctwoparC$ &
        &
        $\coneparoneC$ &
        $\conepartwoF$ &
        $\coneparthrC$ &
        &
        $\czerparC$ \\
        3D &
        7 & 7 & 
        5 & 5 & 5 & 5 &
        3 & 3 &
        4 \\
        2D & 
        12 & 8 &
        9 & 8 & 7 & 6 & 
        6 & 4 &
        6
    \end{tabular}
    \caption{How points with two / one / zero pins and free lines can be observed in the three views of a reduced minimal {\PLoneP} (up to permuting the views).
    The rows ``3D'' and ``2D'' show the degrees of freedom of each local feature in $3$-space and in the three views.
    The row ``\#'' fixes notation for a signature introduced in Section~\ref{sec:balancedPL1Ps}.
    }
    \label{tab:localFeatures}
\end{table}

\section{Balanced PL1Ps}
\label{sec:balancedPL1Ps}
A reduced minimal {\PLoneP} in three views is uniquely determined by a \emph{signature}, a vector consisting of 27 numbers $(\csev, \ldots, \cthr, \ctwofull, \ctwoparA, \ldots, f)$, that specifies how often each local feature occurs in space and how often it is observed in a certain way by the cameras. By Theorem~\ref{thm:reducedMinimalLooks},
the local features in such a {\PLoneP} are free lines or points with at most seven pins. We denote by $f$ the number of free lines
and write $\cthr, \cfou, \ldots, \csev$ for the numbers of points with three, four, $\ldots$, seven pins.
By Theorem~\ref{thm:reducedMinimalLooks},
these local features are completely observed by the cameras.
The row ``\#'' in Table~\ref{tab:localFeatures}
shows our notation 
for the numbers of points with zero, one or two pins that are viewed in a certain way.
For instance, $\ctwofull$ counts how many points with two pins are completely observed by the cameras.
Moreover, $\ctwoparA, \ctwoparB, \ctwoparC$
are the numbers of points with two pins that are partially observed like \includegraphics[width=0.09\textwidth]{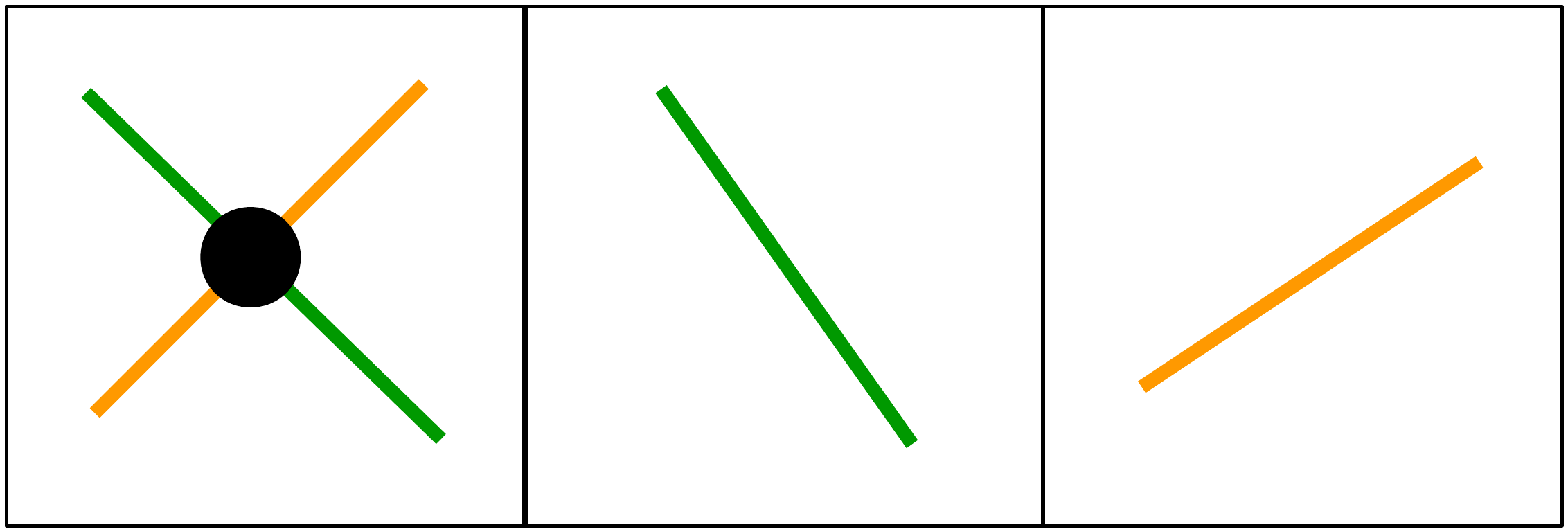} or \includegraphics[width=0.09\textwidth]{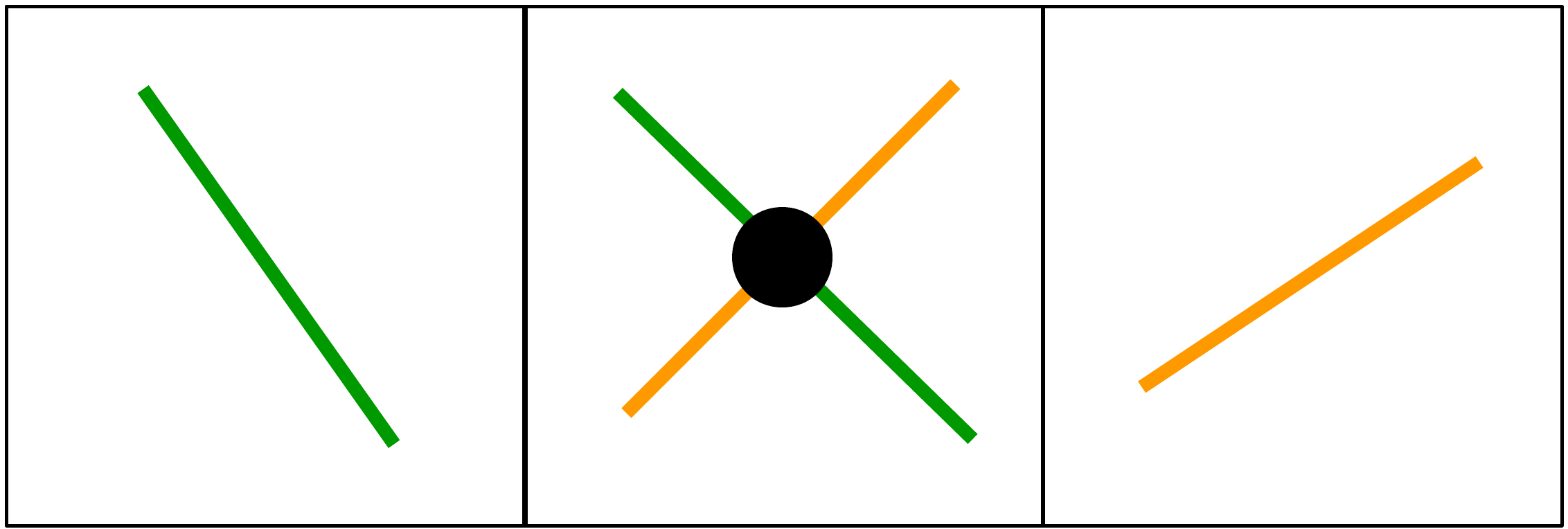} or \includegraphics[width=0.09\textwidth]{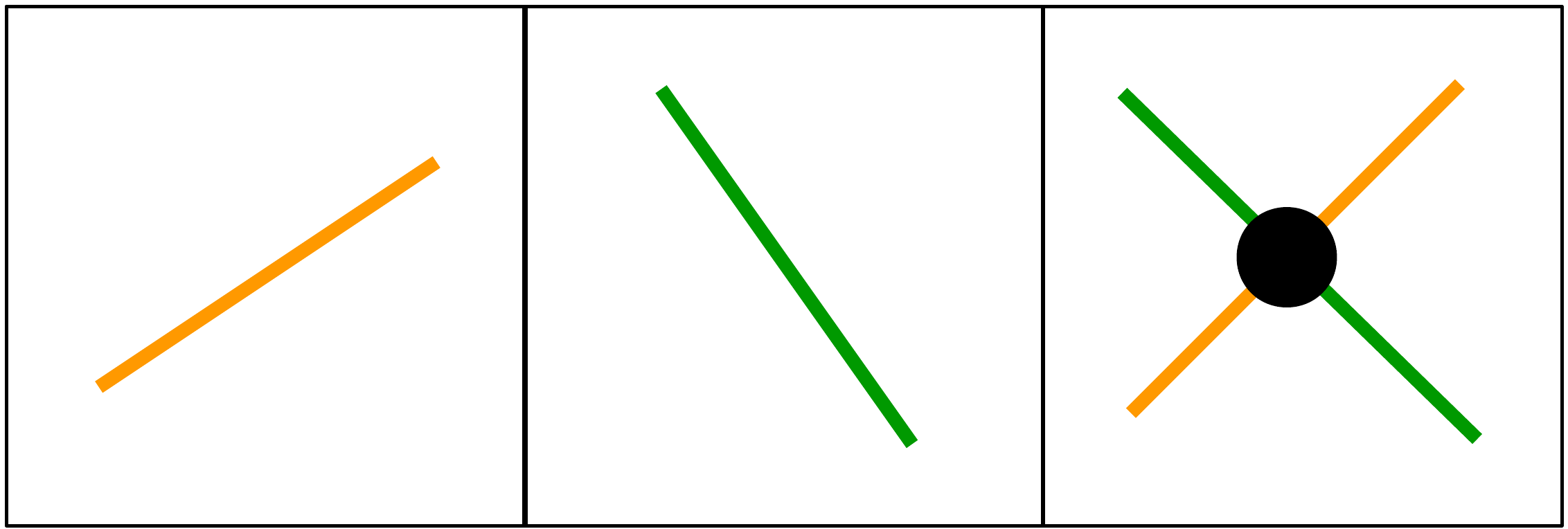}.
Here the upper index $a,b,c$ 
distinguishes the three different permutations of this local feature in the three views (note: as the two pins can be relabeled, there are only three and not six permutations).
Similarly, upper indices distinguish different permutations of partially viewed points with at most one pin; see Table~\ref{tab:localFeatures}.
We also note that assigning arbitrary $27$ non-negative integers to $\csev, \ldots, f$ describes a unique {\PLoneP} in three views, which is reduced by construction (see Thm.~\ref{thm:uniqueReduced} and~\ref{thm:reducedMinimalLooks}) but not necessarily minimal.

Due to Lemma~\ref{lem:balancedPlusDominant}, every minimal {\PLoneP} $(\PLP)$ is balanced, \ie it satisfies $\dim (\PplI \times \cams{m}) =  \dim ( \YplIO)$.
To compute the dimension of $\PplI$, we need to know the degrees of freedom of each local feature in $3$-space. 
For free lines and points with at most two pins, this is given in the row ``3D'' in Table~\ref{tab:localFeatures}.
More generally, a point in space with $k$ pins has $3+2k$ degrees of freedom.
Hence, a reduced minimal {\PLoneP} in three views satisfies
{\footnotesize
\begin{align*}
    \begin{split}
    \dim (\PplI)
    &= 17 \csev + 15 \csix + 13 \cfiv + 11 \cfou + 9 \cthr
     + 7 (\ctwofull + \ctwoparA + \ctwoparB + \ctwoparC)
     \\ &+ 5 (\conefull + \coneparoneA  + \coneparoneB + \coneparoneC + \conepartwoA + \ldots + \conepartwoF + \coneparthrA + \coneparthrB + \coneparthrC)
     \\&+ 3(\czerfull + \czerparA + \czerparB + \czerparC)
     + 4f.
     \end{split}
\end{align*}
}

\noindent
Similarly, the degrees of freedom of each local feature in the three views are shown in row ``2D'' in Table~\ref{tab:localFeatures}.
For instance, if a point with two pins is viewed like
\includegraphics[width=0.09\textwidth]{pix/pin2permA.pdf},
then it has eight degrees of freedom in the three views: 
$2\textcolor{forest}{+1}\textcolor{orange}{+1}$ in the first view, \textcolor{forest}{2} in the second view, and \textcolor{orange}{2} in the third view.
Since a point with $k$ pins for $k = 3, \ldots, 7$ is completely observed by the cameras, 
it has $3(2+k)$ degrees of freedom in the three views.
Therefore, we have
{\footnotesize
\begin{align}
    \label{eq:dim2D}
    \begin{split}
    \dim &(\YplIO)
    = 27 \csev + 24 \csix + 21 \cfiv + 18 \cfou + 15 \cthr
     + 12 \ctwofull + 8(\ctwoparA + \ctwoparB + \ctwoparC)
     \\ &+ 9 \conefull + 8(\coneparoneA  + \coneparoneB + \coneparoneC) + 7(\conepartwoA + \ldots + \conepartwoF) + 6(\coneparthrA + \coneparthrB + \coneparthrC)
     \\&+ 6\czerfull + 4(\czerparA + \czerparB + \czerparC)
     + 6f. 
     \end{split}
\end{align}
}

\noindent
As $\dim(\cams{3}) = 11$, the balanced equality
$\dim (\PplI \times \cams{m}) =  \dim ( \YplIO)$
for a reduced minimal {\PLoneP} in three views is 
$11 = \dim ( \YplIO) - \dim (\PplI )$, \ie
{\footnotesize
\begin{align}
    \label{eq:balanced}
    \begin{split}
    11
    &= 10 \csev + 9 \csix + 8 \cfiv + 7 \cfou + 6 \cthr
     + 5 \ctwofull + (\ctwoparA + \ctwoparB + \ctwoparC)
     \\ &+ 4 \conefull + 3(\coneparoneA  + \coneparoneB + \coneparoneC) + 2(\conepartwoA + \ldots + \conepartwoF) + (\coneparthrA + \coneparthrB + \coneparthrC)
     \\&+ 3\czerfull + (\czerparA + \czerparB + \czerparC)
     + 2f.
     \end{split}
\end{align}
}

\noindent
The linear equation~\eqref{eq:balanced} has 845161 non-negative integer solutions\footnote{See SM for details on how to solve it.}. Each of these is a signature that represents a {\PLoneP} in three views which is reduced and balanced.
Thus, it remains to check which of the 845161 signatures represent \emph{minimal} \PLoneP{}s.

Some of the 845161 solutions of~\eqref{eq:balanced} yield \emph{label-equivalent} \PLoneP{}s, \ie \PLoneP s which are the same up to relabeling the three views.
It turns out that there {\bf 143494} such label-equivalence classes of \PLoneP{}s given by solutions to~\eqref{eq:balanced}\footnote{See SM for details on how to compute this.}.
So all in all, we have to check 143494 \PLoneP{}s for minimality,
namely one representative for each label-equivalence class.
\section{Camera-Minimal PL1Ps}
\label{sec:cameraminimal}
\noindent As in the case of minimal problems, we can understand all camera-minimal \PLoneP{}s from the reduced ones (see also Theorem~\ref{thm:cameraMinimalReduction}).
\begin{theorem}
If a \PLoneP{} is reducible to another \PLoneP{}, 
then either none of them is camera-minimal or both are camera-minimal.
In the latter case, their camera-degrees are equal.
\end{theorem}
In order to understand how reduced camera-minimal \PLoneP s look, in comparison to reduced minimal \PLoneP s as described in Theorem~\ref{thm:reducedMinimalLooks},
we define a pin to be \emph{dangling} if it is viewed by exactly one camera. Dangling pins are not determined uniquely by the camera observations, and hence they appear in \PLoneP{}s that are camera-minimal but not minimal.
\begin{theorem}
\label{thm:danglingPins}
The local features of a reduced camera-minimal \PLoneP{} in three views are viewed as described in Theorem~\ref{thm:reducedMinimalLooks} plus as in the following three additional cases:

\vspace*{1ex}

\begin{tabular}{ccc}
      \includegraphics[width=0.15\textwidth]{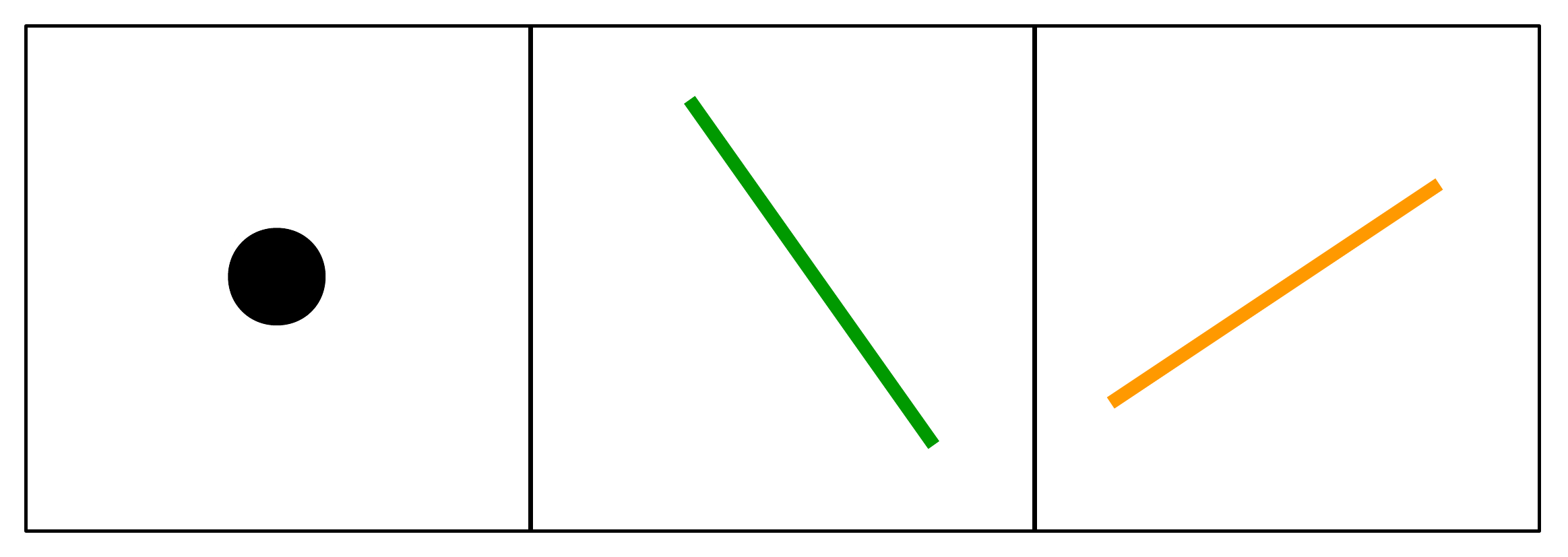}
      &
       \includegraphics[width=0.15\textwidth]{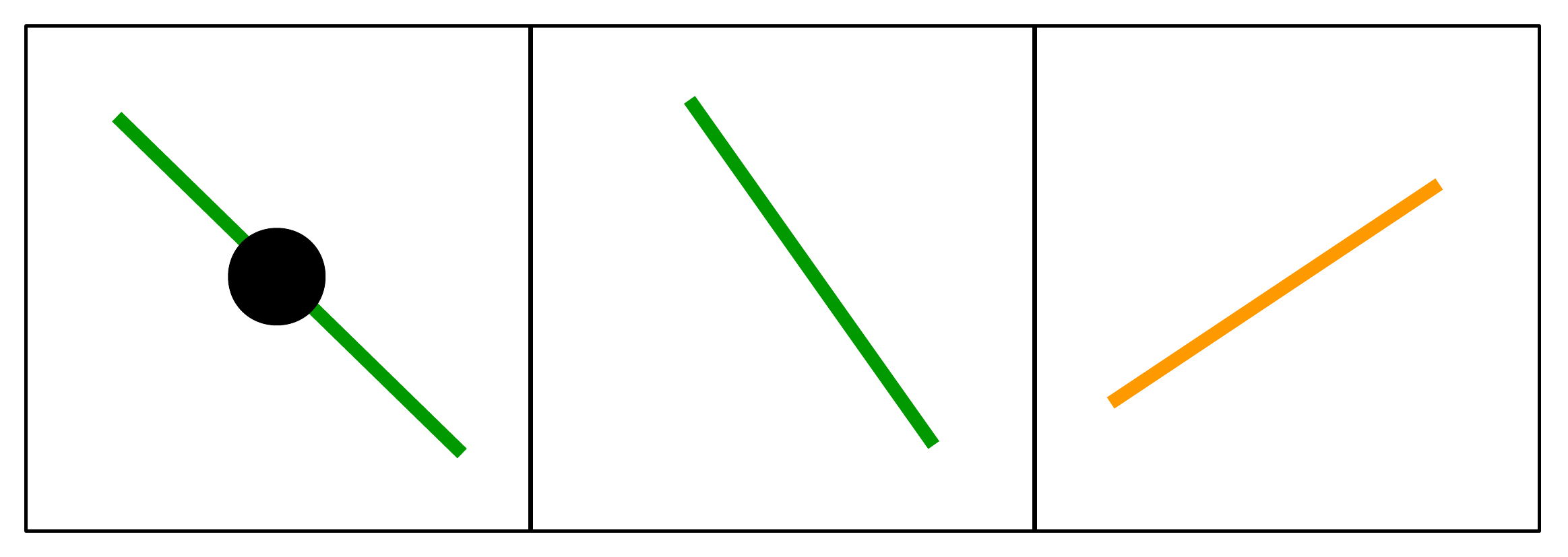}
       &
       \includegraphics[width=0.15\textwidth]{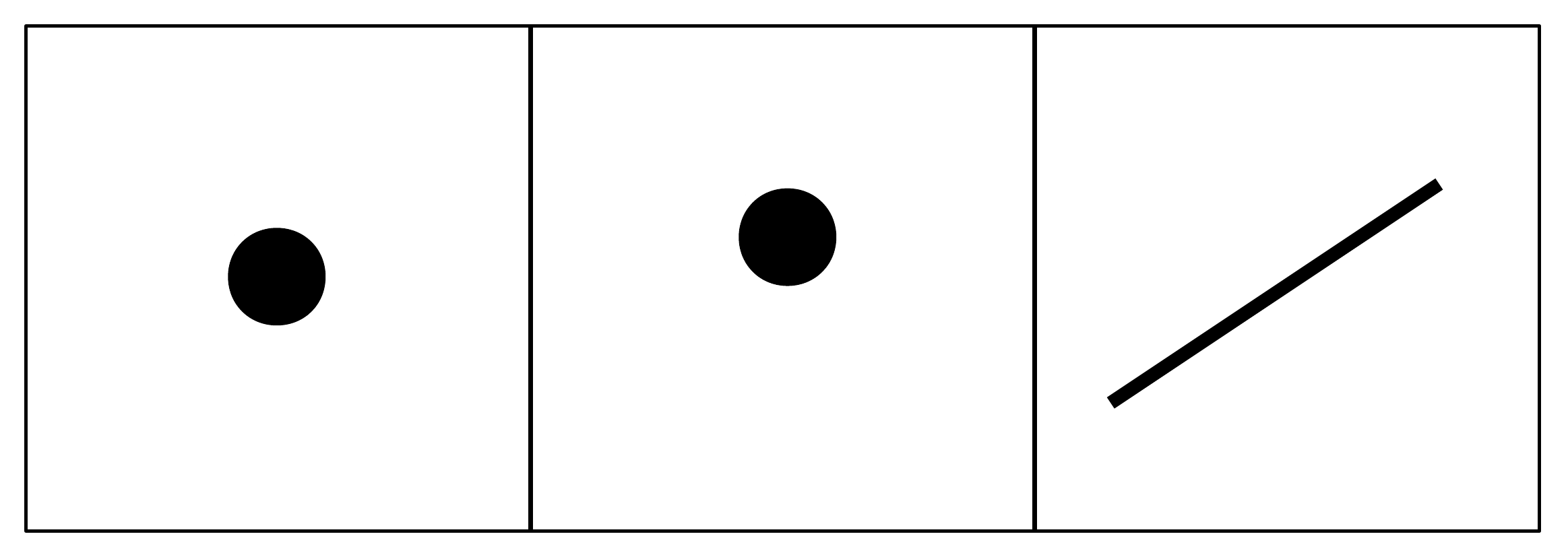}
       \\
    point with two pins,
    &
    point with two pins,
    &
    point with one pin,
    \\
    both are dangling
    &
    one of which is dangling
    &
    which is dangling
\end{tabular}
\end{theorem}

\begin{remark}
For a dangling pin $L$ of a reduced camera-minimal \PLoneP{},
the point $X$ incident to the pin $L$ is uniquely reconstructible.
 Since $L$ is viewed by exactly one camera, it belongs to the planar pencil of lines which are incident to $X$ and have the same image as $L$. Thus we see that $L$ is not uniquely reconstructible from its image.
\end{remark}

\noindent
The next theorem relates minimal and camera-minimal \PLoneP{}s. By adding more constraints to images, we make configurations in space uniquely reconstructible.
\begin{theorem}
\label{thm:camMinLiftToMin}
The following replacements in images lift a reduced camera-minimal \PLoneP{} in three views to a reduced minimal \PLoneP{} (cf. Thm.~\ref{thm:danglingPins} and Table~\ref{tab:localFeatures}):

\noindent
\PDoDt{.09} $\mapsto$ \PotDoDt{.09}, \,
\PoDoDt{.09} $\mapsto$ \PotDoDt{.09}, \,
\PPDo{.09} $\mapsto$ \PoPDo{.09} or 
\PPDo{.09} $\mapsto$ \PPoDo{.09} 

\noindent
Moreover, the camera-degrees of both \PLoneP{}s are the same.
\end{theorem}

\noindent
This  has two important implications for the classification of (camera-)minimal \PLoneP{}s. 
First, reversing the replacements in Theorem~\ref{thm:camMinLiftToMin} transforms each reduced camera-minimal \PLoneP{} in three views into a \emph{terminal} \PLoneP{} of the same camera-degree.

\begin{definition}\label{def:terminal} \rm
We say that a camera-minimal \PLoneP{} in three views is \emph{terminal} 
if it is reduced 
and does not view local features like
\PoDoDt{.09} 
or \PotDoDt{.09}
or \PoPDo{.09}.
\end{definition}

\noindent
Hence, to classify \emph{all} camera-minimal \PLoneP{}s in three views, it is enough to find the terminal ones. 
Secondly, Theorem~\ref{thm:camMinLiftToMin} implies for minimal \PLoneP{}s the following.

\noindent
\begin{corollary}\label{cor:swap}
Consider a  minimal \PLoneP{} in three views. 
After replacing a single occurrence of \PoPDo{.09} with \PPoDo{.09} (or the other way around),
the resulting \PLoneP{} is minimal and has the same degree.
\end{corollary}

\noindent
At the end of Section~\ref{sec:balancedPL1Ps}, we defined two \PLoneP{}s to be label-equivalent if they are the same up to relabeling the views.
We note that the swap described in Corollary~\ref{cor:swap} does \emph{not} preserve the label-equivalence class of a \PLoneP{}.
Instead, we say that two \PLoneP{}s in three views are \emph{swap\&label-equivalent} if one can be transformed into the other by relabeling the views and applying (any number of times) the swap in Corollary~\ref{cor:swap}.
We conclude that either all \PLoneP{}s in the same swap\&label-equivalence class are minimal and have the same degree, or none of them is minimal.
Moreover, the lift in Theorem~\ref{thm:camMinLiftToMin} yields the following.

\begin{corollary}\label{cor:swap-label-equivalence}
The swap\&label-equivalence classes of reduced minimal \PLoneP{}s in three views are in
a camera-degree preserving
one-to-one correspondence with the label-equivalence classes of terminal camera-minimal \PLoneP{}s in three views. 
\end{corollary}

\noindent
Hence, we do not have to check minimality for all 143494 label-equivalence classes of \PLoneP{}s given by solutions to~\eqref{eq:balanced}, that we found at the end of Section~\ref{sec:balancedPL1Ps}.
Instead it is enough to consider the swap\&label-equivalence classes of the solutions to~\eqref{eq:balanced}.
It turns out that there are \textbf{76446} such classes\footnote{See SM for details on how to compute this.}.
So to find \emph{all} (camera-)minimal \PLoneP{}s in three views, we only have to check 76446 \PLoneP{}s for minimality, namely one representative for each of the swap\&label-equivalence classes.

\noindent
Finally, we present the analog to Theorem~\ref{thm:uniqueReduced} and describe how all camera-minimal \PLoneP{}s are obtained from the reduced camera-minimal ones.

\begin{theorem}
\label{thm:cameraMinimalReduction}
A camera-minimal \PLoneP{} in three views is reducible to a unique reduced \PLoneP{}. The corresponding projection forgets:

\noindent $\bullet\ $everything that is forgotten in Theorem~\ref{thm:uniqueReduced}

\noindent $\bullet\ $every line (free or pin) that is not observed in any view

\noindent $\bullet\ $every free line that is observed in exactly one view to reduce
\includegraphics[width=0.09\textwidth]{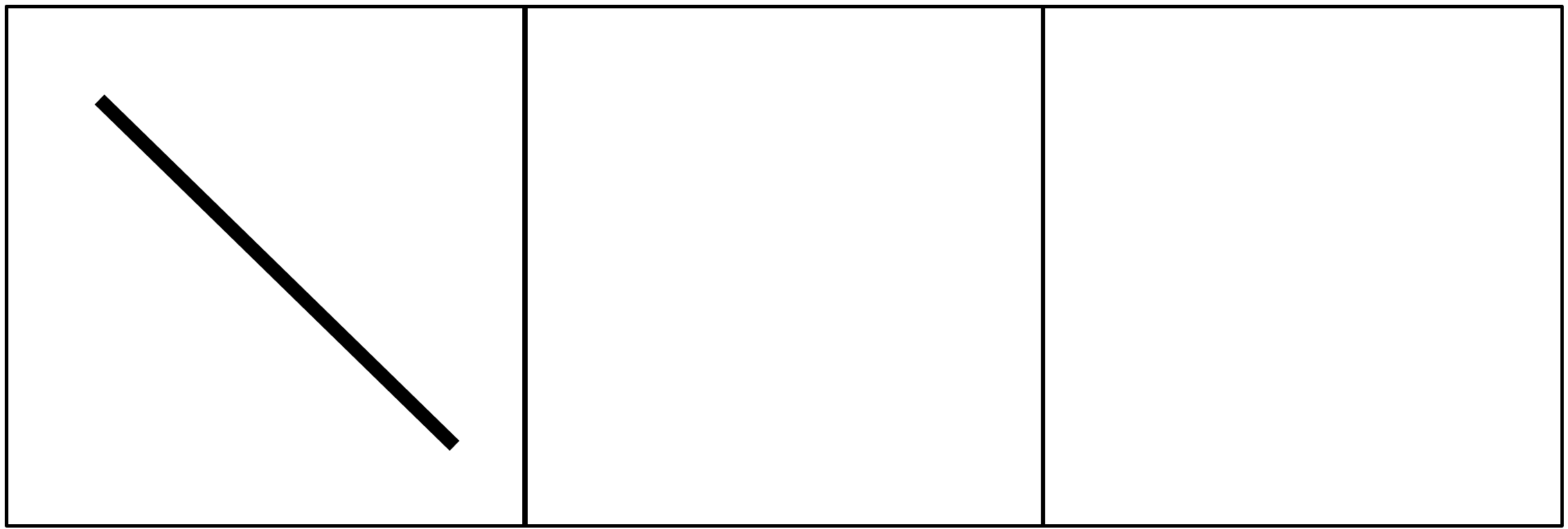} to
\includegraphics[width=0.09\textwidth]{pix/empty.pdf}

\noindent $\bullet\ $every pin that is observed in exactly one view such that the view also observes 
\linebreak[4] \indent \hspace{-3.4mm}
the point of the pin (it does not matter if the other two views observe the point 
\linebreak[4] \indent \hspace{-3.4mm}
or not, but they must not see the line), e.g. 
\includegraphics[width=0.09\textwidth]{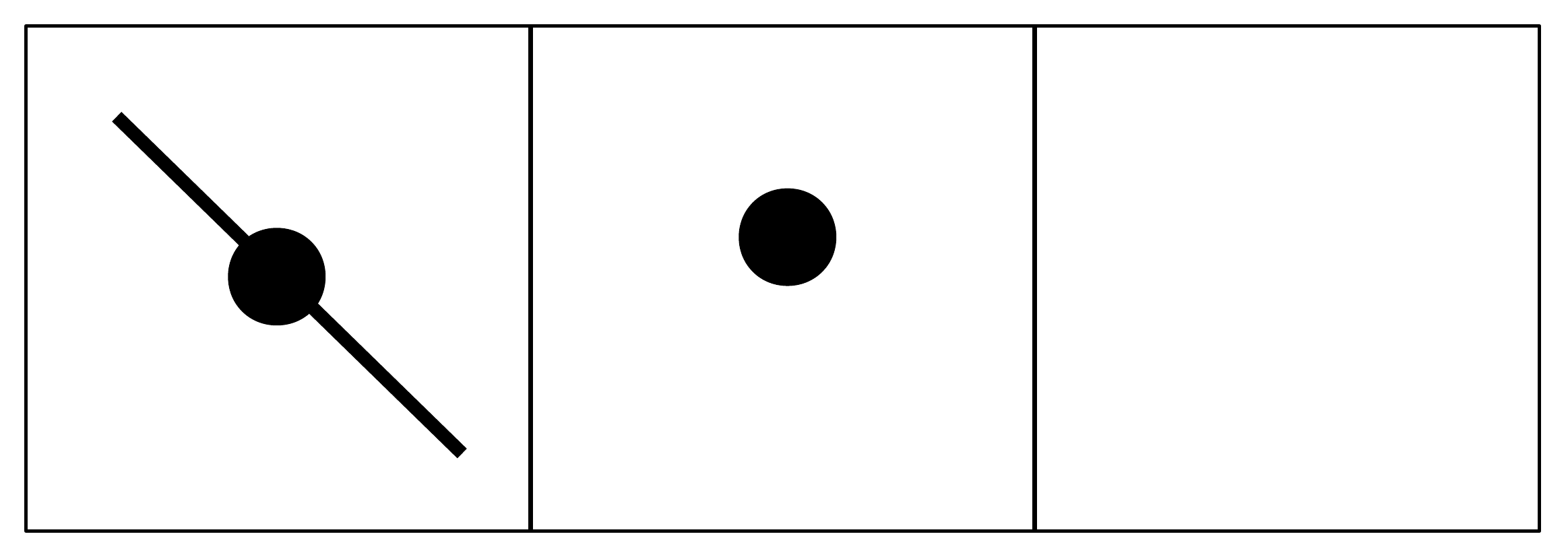}
is reduced to
\includegraphics[width=0.09\textwidth]{pix/pin0param2.pdf}

\noindent $\bullet\ $every point without pins that is observed in at most one view, e.g.
\includegraphics[width=0.09\textwidth]{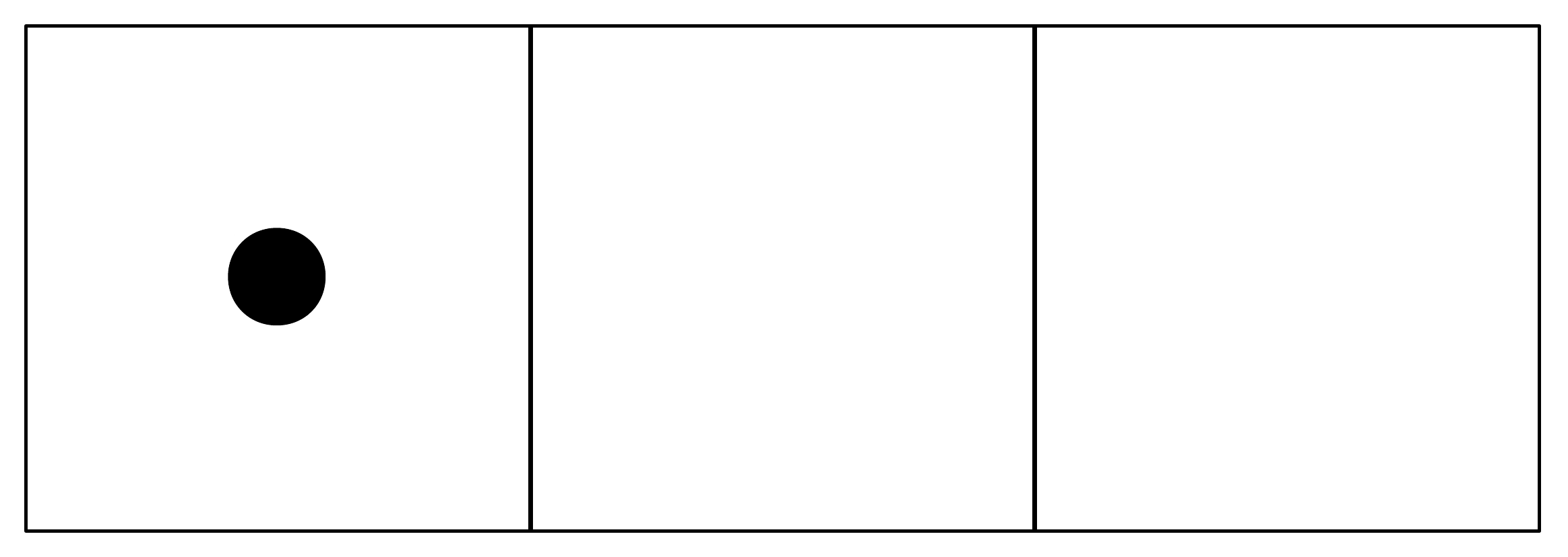} is 
\linebreak[4] \indent \hspace{-3.4mm}
reduced to \includegraphics[width=0.09\textwidth]{pix/empty.pdf}

\noindent $\bullet\ $every point that has exactly one pin if the point is not observed in any view,
\linebreak[4] \indent \hspace{-3.4mm}
e.g. a pin viewed like \includegraphics[width=0.09\textwidth]{pix/freeline.pdf} becomes a free line viewed like \includegraphics[width=0.09\textwidth]{pix/freeline.pdf}

\noindent $\bullet\ $every point together with its single pin if it is viewed like \includegraphics[width=0.09\textwidth]{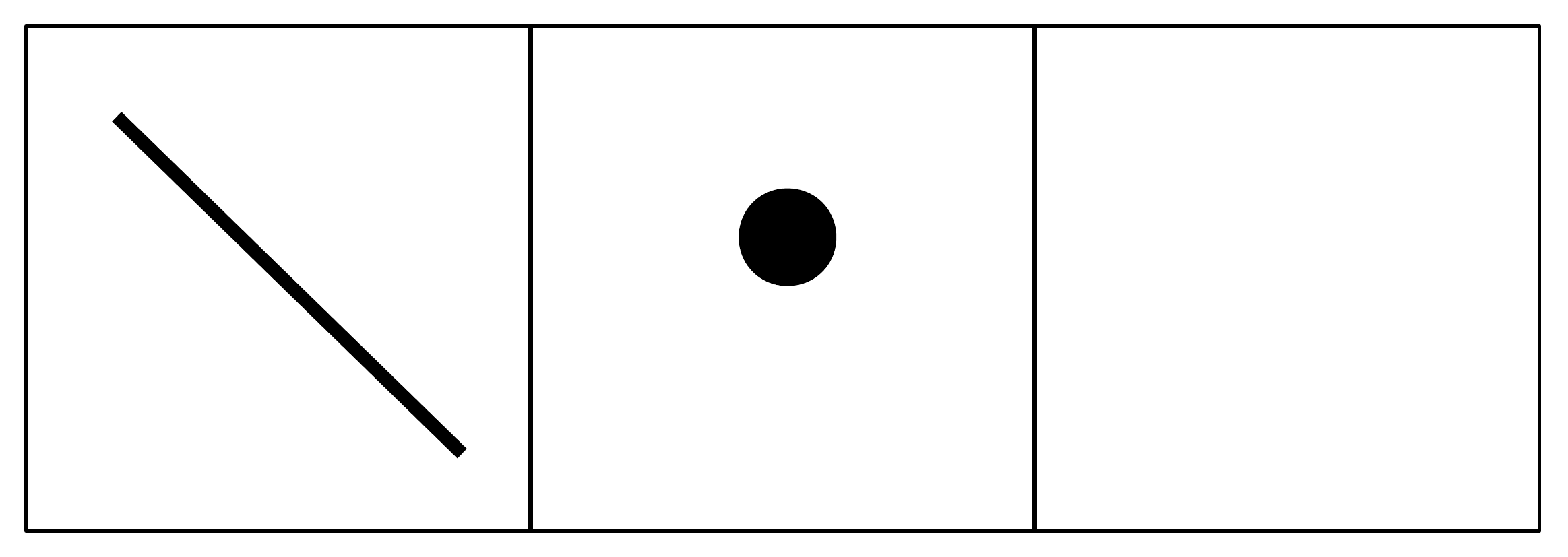}  to get \includegraphics[width=0.09\textwidth]{pix/empty.pdf}

\noindent $\bullet\ $every point together with all its pins if the point has at least two pins and the
\linebreak[4] \indent \hspace{-3.4mm}
point is not observed in any view, e.g. \includegraphics[width=0.09\textwidth]{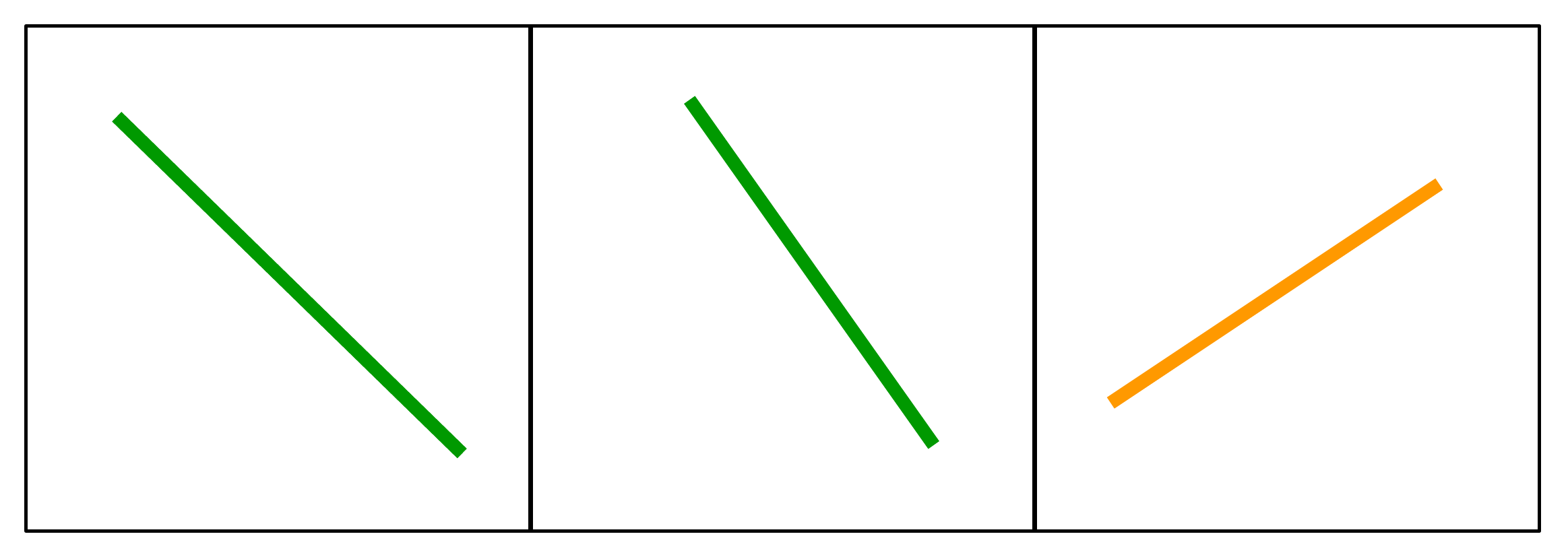} is reduced to \includegraphics[width=0.09\textwidth]{pix/empty.pdf}
\end{theorem}

\section{Checking minimality}
\label{sec:minimality}

To show that a balanced point-line problem $(\PLP)$ is minimal, it is equivalent to show that the Jacobian of the joint camera map $\Phi_{\PLP }$ at some point $(X,P)\in \PplI \times \cams{m}$ has full rank, i.e.\ rank given by the formula in equation~\eqref{eq:dim2D}.
This follows from Lemma~\ref{lem:balancedPlusDominant}, as explained in~\cite{PLMP}. 
On the implementation level, this minimality criterion requires writing down local coordinates for the various projective spaces and Grassmannians. To take advantage of fast exact arithmetic and linear algebra, we ran each test with random inputs $(X,P)$ over a finite field $\FF_q$ for some large prime $q$.
We observe that \emph{false positives}\footnote{Since we are testing minimality, being minimal is the positive outcome. See SM for detailed explanation why false positives cannot occur.} for these tests are impossible. To guard against \emph{false negatives}, we re-run the test on remaining non-minimal candidates for different choices of $q.$ Moreover, as a byproduct of our degree computations, we obtain yet another test of minimality, following the same procedure as \cite[Algorithm 1]{PLMP}.

\noindent
The computation described above detects non-minimality for 2878 of the 143494 label-equivalence classes of \PLoneP{}s given by solutions to~\eqref{eq:balanced}.
Among the 76446 swap\&label-equivalence classes, 1871 are not minimal.

\begin{result}
\label{thm:minimalProblemsFound}
There are 
    
\noindent $\bullet\ $ $\boldsymbol{140616}=143494-2878$ reduced minimal \PLoneP{}s and 

\noindent $\bullet\ $  $\boldsymbol{74575}=76446-1871$ terminal camera-minimal \PLoneP{}s 

\noindent in three calibrated views, up to relabeling cameras.
\end{result}
\section{Computing degrees}\label{sec:degrees}
From the perspective of solving minimal problems, it is highly desirable to compute all degrees of our minimal \PLoneP{}s. In particular, we wish to identify problems with small degrees that may be of practical interest. Since some problems in this list are known from prior work~\cite{Kileel-MPCTV-2016} to have large degrees ($>1000$), our main technique is a monodromy approach based on numerical homotopy continuation. Our implementation in Macaulay2~\cite{M2,Duff-Monodromy} is similar to that used in prior work~\cite{PLMP}.
The next result shows that there are many interesting problems with small degrees that are solvable with existing solving technology~\cite{Larsson-AGMP-2018,Fabri-ArXiv-2019}.
\begin{result}
\label{result:smalldegs}
There are $759$ (up to relabeling  cameras) terminal camera-minimal \PLoneP{}s in three calibrated views with camera-degree less than $300$. Their camera-degree distribution is shown in \Cref{tab:smalldegs}(a).
\end{result}
\begin{table}[t]
\centering
(a) \begin{tabular}{|c|ccccccccccc|}
\hline 
camera-degree & 64& 80&  144& 160& 216& 224& 240& 256& 264& 272& 288 \\ \hline
\# problems & 13& 9&  3& 547& 7& 2& 159& 2& 2& 11& 4 \\
\hline
\end{tabular}
\\[1.2em]
(b) \begin{tabular}{|c|ccccccccccccccc|}
\hline
camera-degree & 80& 160& 216& 240& 256& 264& 272& 288& 304& 312& 320& 352& 360& 368& 376 \\
\hline
\# problems & 9& 173& 4& 80& 2& 2& 2& 1& 5& 2& 213& 3& 9& 3& 1\\
\hline \hline
camera-degree & 384& 392& 400& 408& 416& 424& 432& 448& 456& 464& 472& 480& 488& 496 & \\
\hline
\# problems & 2& 9& 14& 2& 6& 10& 2& 7& 11& 4& 1& 96& 12& 9 & \\
\hline 
\end{tabular}

\vspace*{1em}

\caption{Distribution of camera-degrees of terminal camera-minimal \PLoneP{}s in three calibrated views with: (a) camera-degree less than $300$, (b) at most one pin per point and camera-degree less than 500.}
\label{tab:smalldegs}
\end{table}
It is also interesting to look at problems with simple incidence structure, since they are easier to detect in images. This motivates the following.
\begin{definition}\label{def:PLkP}\rm
We say that a point-line problem is a \emph{\PLkP{\kappa}} if each line in 3D is incident to at most $\kappa$ points.
\end{definition}
Definition~\ref{def:PLkP} generalizes Definition~\ref{def:PL1P} of \PLoneP{}s to get a hierarchy of subfamilies of point-line problems: $$\PLzeroP  \subset \PLoneP \subset  \PLkP{2} \subset \dots$$

\noindent
For instance, \PLzeroP{}s  consist of free points and free lines only\footnote{We note that reduced minimal \PLzeroP{}s are terminal.}. The family of \PLoneP{}s contains many problems involving at most one pin per point---see Result~\ref{res:p1l1p} and \Cref{tab:smalldegs}(b). Such features are readily provided by the  SIFT~\cite{DBLP:journals/ijcv/Lowe04} detector. \PLoneP{}s can also include features with two pins per point, which are readily provided by the LAF~\cite{Matas-ICPR-2002} detector, which can also be used to get $\PLkP{2}$s if all 3 LAF points are used. More complex incidences (\eg $\PLkP{3}$s) can be obtained from line t-junction detectors~\cite{DBLP:journals/ijcv/XiaDG14}.

\begin{result}
\label{res:p1l1p}
There are $9533$ (up to relabeling cameras) terminal camera-minimal \PLoneP{}s in three calibrated views which have at most one pin per point.
\textbf{694} of them have camera-degree less than \textbf{500}. 
Their camera-degree distribution is shown in~\Cref{tab:smalldegs}(b).
\end{result}

\begin{result}
\label{res:pl0p}
There are $51$ (up to relabeling  cameras) reduced minimal \PLzeroP{}s in three calibrated views. They are depicted together with their degrees in \Cref{tab:PL0P}.
\end{result}

\noindent Note that there are four problems in \Cref{tab:PL0P} that are \emph{extensions} of the classical minimal problem of five points in two views. This implies that the relative pose of the two cameras can be determined from the five point correspondences (highlighted in red). As to the remaining camera, each of these four problems can be interpreted as a \emph{camera registration} problem (the first one is known as P3P~\cite{DBLP:conf/cvpr/KneipSS11}): given a set of points and lines in the world and their images for a camera, find that camera pose.
Note that the solution counts indicate that there are 8, 4, 8, and 8 solutions to the corresponding four camera registration problems. Similar degrees were previously reported for camera registration from 3D points and lines for perspective and generalized cameras~\cite{Dhome-PAMI-1989,Chen-ICCV-1990,Ramalingam-ICRA-2011,Miraldo-TC-2015,Miraldo-ECCV-2018}.

\begin{result}\label{res:registration}
We determined all \PLoneP{}s in three calibrated views that are extensions of the five-points minimal problem. Of them, up to relabeling cameras,
    
    \noindent $\bullet\ $ $6300$ are reduced minimal,
    
    \noindent $\bullet\ $ $61$ of the $6300$ correspond to camera registration problems,\footnote{See SM for the table of all extensions leading to registration problems.} and
    
    \noindent $\bullet\ $ $3648$ are terminal camera-minimal.
\end{result}

\begin{table}[t]
    \centering
    \resizebox{1.0\textwidth}{!}{
    \begingroup
    \tikzstyle{every node}=[circle, draw, fill=black,inner sep=0pt, minimum width=6pt]
    \tikzstyle{redpoint}=[circle, draw, fill=red,inner sep=0pt, minimum width=6pt]
    \tikzstyle{ghost}=[circle, draw, fill=white, opacity=0,inner sep=0pt, minimum width=4pt]
    \input{tables/pl0p-table-redpoint.tex}
    \endgroup
    }
    
    \vspace*{1ex}
    
    \caption{Reduced minimal \PLzeroP{}s and their degrees. Points not visible in a given view are indicated in grey. Five-point subproblems are indicated in red. \vspace{-4mm}} 
    \label{tab:PL0P}
\end{table}

\vspace{-1mm}

\section{Conclusion}
In the context of configurations of points and lines, we have constructed an explicit classification of reduced minimal and camera-minimal problems in three calibrated views with the restriction that lines contain at most one point. Our results are rigorous: they rely on theorems stated in this article with proofs provided in the SM and assisted by computational enumeration programs (the code will be publicly available).

The number of (camera-)minimal problems in our classification is large. Apart from constructing a database of all these problems, we provide tables for interesting subfamilies where the number of the problems is relatively small. \Cref{tab:PL0P} in this article lists all such problems that are \PLzeroP{s} and there is a table in the SM showing all extensions of the bifocal five-point problem that can be interpreted as problems of camera registration.

Another part of our computational effort focused on determining algebraic degrees of the (camera-)minimal problems. The degree of a  problem provides a measure of complexity of a solver one may want to construct. The smaller the degree, the more plausible it is that a problem could be used in practice: \Cref{tab:smalldegs} shows the degree distributions for problems of degree less than 300. Computing exact degrees for all problems is an ongoing work.
\clearpage

{\small
\bibliographystyle{splncs}
\bibliography{bib/local}
}

\newpage
\begin{center}
    \textbf{{\Large \PLoneP{} - Point-line Minimal Problems under Partial Visibility in Three Views \\[1ex] --- SM ---}} \\
\bigskip
\ifarxiv\else
Paper ID \ECCVSubNumber
\fi
\end{center}

\noindent Here we give proofs of all results from the main paper, some additional results for camera registration, and other details.
On the last page, we provide a glossary of assumptions, properties and concepts used through the whole article; see~\Cref{sec:glossary}.

\section{Note on partial visibility in two views}

Let us recall that we consider reduced minimal and reduced camera-minimal \PLoneP{}s in \emph{three views}, since there is only one such \PLoneP{} in two views, namely the five-points problem. 

To argue this, we first notice that free lines cannot occur in a reduced \PLoneP{} in two views. Indeed, if there was a free line -- no matter if it is observed in both views, one view, or not at all -- we could forget the free line both in 3D and in the images to reduce the \PLoneP{} (i.e. forgetting the free line satisfies Definition~\ref{def:reduced-PL1P} of reducibility).

Secondly, we argue that pins cannot appear in a reduced \PLoneP{} in two views.
To see this we distinguish several cases, depending on how often a hypothetical pin is observed in the views.

\begin{itemize}
    \item If there exists a pin in 3D that is not observed in any of the two views, then the \PLoneP{} would be reducible by simply forgetting that pin.
    \item If a pin in 3D is observed in exactly one of the two views, it could either appear as a pin or as a free line in that view (\ie depending on if that view also observes the point of the pin or not).
    \begin{enumerate}
        \item If the single view seeing the pin also observes its point, the \PLoneP{} is reducible by forgetting the pin.
        \item If the single view seeing the pin does not observe its point, that view cannot observe any of the other pins of that point either (due to our assumptions on $\mathcal{O}$ at the end of Definition~\ref{def:PLP}). 
        Hence, no matter how the other view observes that point and its other pins, the \PLoneP{} is reducible by forgetting the point together with all of its pins.
    \end{enumerate}
    \item If a pin in 3D is observed in both views, it could appear as either a pin or a free line in either of the views.
    \begin{enumerate}
        \item If both views also observe the point of the pin, the \PLoneP{} is reducible by forgetting the pin.
        \item If at least one of the two views does not observe the point of the pin, then we may argue as in case~2 above that such a view cannot observe any of the other pins of that point either. Once again, the \PLoneP{} is reducible by forgetting the point together with all of its pins.
    \end{enumerate}
\end{itemize}

Finally, a \PLoneP{} which observes a point without pins in at most one of its views is reducible by forgetting that point.
All in all, we conclude that reduced \PLoneP{}s in two views can only consist of points without pins which are observed in both views.
The only such (camera-)minimal problem is the five-points problem.

\section{Note on registration problems}
Among the minimal problems we discovered that there are extensions of the classical five-point problem in two views that can be interpreted as camera registration problems; see \Cref{tab:registration}.

\begin{table}[h!]
    \centering
    \resizebox{1.0\textwidth}{!}{
    \begingroup
    \tikzstyle{every node}=[circle, draw, fill=black,inner sep=0pt, minimum width=6pt]
    \tikzstyle{ghost}=[circle, draw, fill=white, opacity=0,inner sep=0pt, minimum width=4pt]
    \input{tables/registration-problems-table.tex}
    \endgroup
    }
    \caption{Camera registration: minimal problems that can be determined as relative pose problems for two views appended with a minimal registration problem for the third camera.} 
    \label{tab:registration}
\end{table}

\begin{remark}
Extension of reduced camera-minimal problems by means of camera registration problems is one simple construction for (camera-)minimal problems in arbitrarily many views.    
\end{remark}

\section{Proofs}

\subsection{Theorems 1 and 4}

We start by investigating the implications of the two conditions in Definition~\ref{def:reduced-PL1P} of reducibility. When obtaining  a new \PLoneP{} $(p',l',\cI',\obs')$ from a given \PLoneP{} $(\PLP)$ by forgetting some of its points and lines, the induced projections $\Pi$ and $\pi$ between the domains and codomains of the joint camera maps yield the commutative diagram in Figure~\ref{fig:commutativeDiagram}. We call $\Pi$ the \emph{forgetting map}, and note that the map $\pi$ between the image varieties is completely determined by the forgetting map $\Pi$. If the forgetting map $\Pi$ satisfies the first condition in Definition~\ref{def:reduced-PL1P} (i.e. for each forgotten point, at most one of its pins is kept), we say that it is \emph{feasible}. The second condition in Definition~5 we refer to as the \emph{lifting property}.

\begin{figure}
    \centering
        \begin{tikzcd}
        \PplI \times \cams{m}
        \arrow[r, dashed, "\Phi~=~\Phi_{\PLP}"] 
        \arrow[d, "\ppi" left]
         &[20mm] \YplIO  
        \arrow[d, "{\pi}"]
         \\
        \cX_{p',l',\cI'} \times \cams{m}
        \arrow[r, dashed, "\Phi'~=~\Phi_{p',l',\cI',\obs'}" below]
        &[20mm]
        \mathcal{Y}_{p',l',\cI',\obs'}
        \end{tikzcd}
    \caption{Two \PLoneP{}s and their joint camera maps related by a forgetting map $\Pi$. The upper \PLoneP{} is reducible to the lower \PLoneP{} iff the forgetting map $\Pi$ is feasible and satisfies the lifting property.}
    \label{fig:commutativeDiagram}
\end{figure}

\begin{lemma}
\label{lem:forget-irreducible}
If the forgetting map $\Pi$ is feasible, then both $\Pi$ and $\pi$ are surjective and have irreducible fibers\footnote{A fiber of a map is the preimage over a single point in its codomain.} of equal dimension.
\end{lemma}

\begin{proof}
We first show that $\Pi$ is surjective and has irreducible fibers of the same dimension. We may assume that $\Pi$ forgets either a single point or a single line, as a every feasible forgetting map is the composition of several feasible forgetting maps which forget either one point or one line. Moreover, the composition of several surjective maps with irreducible fibers of equal dimension is again surjective and has irreducible fibers of the same dimension.


Let us first assume that $\Pi$ forgets a single line. If this line is the pin of a point, each fiber of $\Pi$ is the set of lines in $\PP^3$ through that point; this set is isomorphic to $\PP^2$. If $\Pi$ forgets a free line, its fibers are isomorphic to $\GG_{1,3}$. In both cases, the forgetting map $\Pi$ is surjective.

Now let us assume that $\Pi$ forgets a single point. As $\Pi$ is feasible, this point is incident to at most one line in space. Depending on if it is incident to one or zero lines, the fibers of $\Pi$ are isomorphic to either $\PP^1$ or $\PP^3$. In either case, $\Pi$ is surjective.

To show that $\pi$ is surjective with irreducible fibers of equal dimension, we apply a nearly-identical proof as above to the $m$ factors of $\pi$ corresponding to the different views.
\qed
\end{proof}

For a feasible forgetting map $\Pi$, we write  $\fiberPi$, resp.\ $\fiberpi$, for the dimensions of the fibers of $\Pi$, resp.\ $\pi$. Moreover, we denote by $\cdeg(\PLP)$ the camera-degree of a point-line problem $(\PLP)$. We will also use this notation for point-line problems which are not camera-minimal:  in that case, the camera-degree is either zero (if the joint camera map is not dominant) or $\infty$.

\begin{lemma}\label{lem:diagram-one-way}
Consider Figure~\ref{fig:commutativeDiagram} with a feasible forgetting map $\Pi$.
\begin{enumerate}
    \item  If the upper joint camera map $\Phi$ is dominant, so is the lower one $\Phi'$.
    \item $\cdeg(\PLPprime) \geq \cdeg(\PLP)$.
    \item If $(\PLP)$ is minimal, then 
    $\fiberPi \leq \fiberpi$.
\end{enumerate}
\end{lemma}

\begin{proof}
For a generic image $(x',\ell') \in \mathcal{Y}_{\PLPprime}$ of the lower \PLoneP{}, we consider a generic image $(x,\ell) \in \pi^{-1} (x', \ell')$ of the upper \PLoneP{}. Every solution $S \in \Phi^{-1}(x,\ell)$ of the upper \PLoneP{} yields a solution $\Pi(S)$ of the lower \PLoneP{} with \emph{the same} cameras. This shows the first two parts of the assertion.

For the third part, since the joint camera map $\Phi$ of a minimal \PLoneP{} is dominant, we use part~1 to see that the lower joint camera map $\Phi'$ is also dominant. This implies:
\begin{align*}
    \begin{array}{ccc}
        \dim(\cX_{p',l',\cI'} \times \cams{m}) & \geq & \dim (\mathcal{Y}_{\PLPprime})  \\
        || & & || \\
        \dim (\PplI \times \cams{m}) - \fiberPi & &  \dim(\YplIO) - \fiberpi
    \end{array}
\end{align*}
Since each minimal \PLoneP{} is balanced, i.e. $\dim (\PplI \times \cams{m}) = \dim(\YplIO)$, this concludes the proof.
\qed
\end{proof}

\begin{lemma}
\label{lem:diagram-two-way}
Consider Figure~\ref{fig:commutativeDiagram} with a feasible forgetting map $\Pi$ that satisfies the lifting property.
\begin{enumerate}
    \item $\Phi$ is dominant $\Leftrightarrow$ $\Phi'$ is dominant.
    \item $\cdeg(\PLPprime) = \cdeg(\PLP)$.
    \item $\fiberPi \geq \fiberpi$.
\end{enumerate}
\end{lemma}

\begin{proof}
We can pick a generic image  of the upper \PLoneP{}, by first choosing a generic image $(x',\ell') \in \mathcal{Y}_{\PLPprime}$ of the lower \PLoneP{} and then considering a generic image $(x,\ell) \in \pi^{-1}(x',\ell')$ in its fiber. If $\Phi'$ is dominant, a generic solution $S' \in \Phi'^{-1}(x',\ell')$ of the lower \PLoneP{} can be lifted to a solution $S \in \Phi^{-1}(x, \ell)$ of the upper \PLoneP{} with \emph{the same} cameras. Together with Lemma~\ref{lem:diagram-one-way}, this shows the first two parts of the assertion.

For the third part, we consider a generic element $S' \in \cX_{p',l',\cI'} \times \cams{m}$. By the lifting property, we have that
\begin{align*}
\fiberPi = \dim(\Pi^{-1}(S')) \geq \dim(\pi^{-1}(\Phi'(S'))) = \fiberpi. \hspace{13mm} \qed
\end{align*}
\end{proof}

\noindent
\textbf{Proof of Lemma \ref{lem:degs-equal}.} Let $(\PLP )$ be a minimal \PLoneP{} with camera-degree $d$. Fix $(x,\ell) \in \YplIO$ generic and let $\gamma \left( \Phi_{\PLP}^{-1} (x, \ell )\right) = \{ P_1, \ldots , P_d \}.$ We may assume that $ \Phi_{\PLP}^{-1} (x, \ell )$ is a nonempty finite set; we wish to show it contains $d$ elements. We begin by lifting $P_1,\ldots, P_d$ to partial solutions in $\GG_{1,3}^l \times \cams{m}$, \ie by reconstructing lines in 3D. Our first observation is that each of the $l$ world lines must be viewed at least twice. Indeed, if a line was observed at most once, then even if it is a pin (\ie it passes though a point in 3D) there would be an at least one-dimensional family of world lines yielding the same view (and possibly passing through the pin point). This contradicts our assumption that the \PLoneP{} $(\PLP)$ is minimal.

Now since each world line is observed in at least two views, it can be uniquely recovered from fixed cameras; in other words, each camera solution $P_i$ extends uniquely to a partial solution $(L_i,P_i) \in \GG_{1,3}^l \times \cams{m}$. To conclude that there are exactly $d$ lifts 
\[
\left((X_1, L_1), P_1\right),\ldots , \left((X_d, L_d) , P_d\right) \in \Phi_{\PLP}^{-1} (x, \ell ),
\]
we note that the condition that each $X_i$ is a lifted solution  of the partial solution $(L_i, P_i)$ is given by linear equations depending on $P_i$ and $L_i.$ Thus, there is either a unique or infinitely many such $X_i$, but since the \PLoneP{} $(\PLP)$ is minimal, the latter cannot hold.
\qed

\begin{example}
Lemma~\ref{lem:degs-equal} does \emph{not} hold for point-line problems in general.
Here we present a minimal \PLkP{4} whose degree is double its camera-degree.

Consider the variety $\mathcal{X}_{9,5,\mathcal{I}}$ of point-line arrangements consisting of five free points (red in Figure~\ref{fig:schubert}) and five lines (and four additional black points) where one of the lines (in black) intersects the other four (colorful) lines (in the four black extra points), i.e.\ $$\mathcal{I} = \{(1,1), (2,2), (3,3), (4,4), (1,5), (2,5), (3,5), (4,5)\}.$$ 
We are interested in the \PLkP{4} in two views where both views observe the five free points and the first four (colorful) lines, but not the fifth (black) line  and none of the (black) extra points; see Figure~\ref{fig:schubert}.

\begin{figure}
    \centering
    \begin{tabular}{ccc}
    \includegraphics[height=0.4\textwidth]{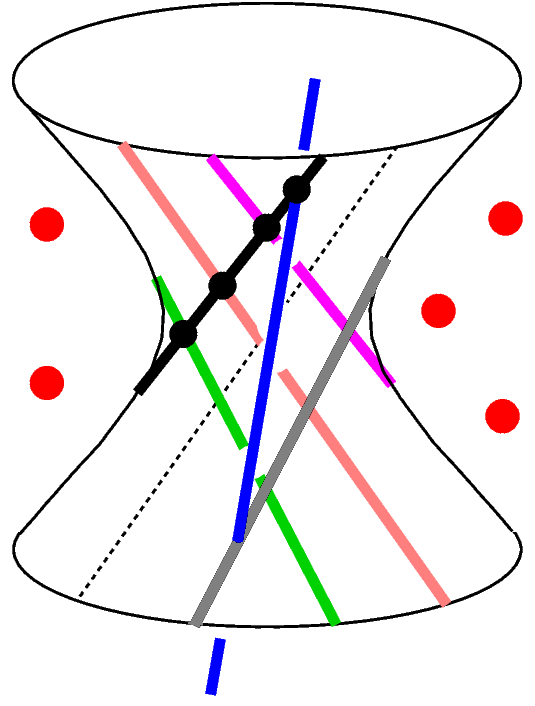} 
    &\hspace*{0.03\textwidth}&
    \includegraphics[height=0.2\textwidth,angle=90]{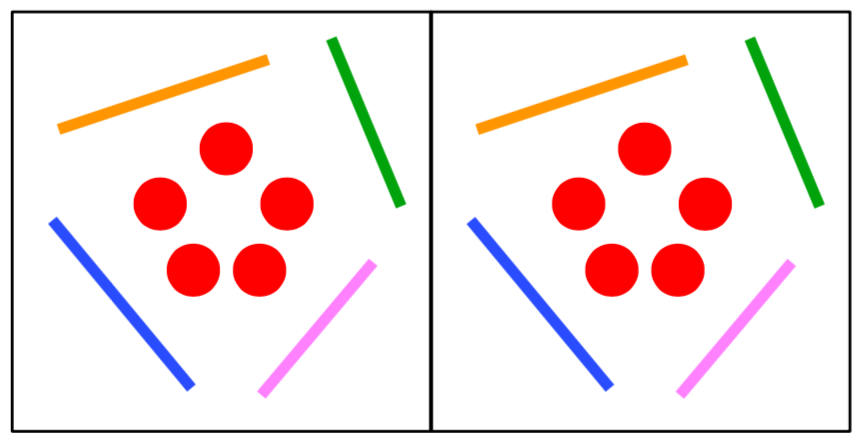}\\
    \parbox{0.45\textwidth}{An arrangement of 5 lines and 9 points in 3D space. The grey line is \emph{not} part of the arrangement.} &&
    \parbox{0.45\textwidth}{A projection to two images that forgets all black features.}
\end{tabular}    
    \caption{The five-point problem in two views can be  combined with the classical Schubert four-line problem to obtain a minimal \PLkP{4} with degree $40$ and camera-degree $20$.}
    \label{fig:schubert}
\end{figure}

Clearly this point-line problem contains the classical five-point problem as a subproblem. In fact, it is reducible to the five-point problem by forgetting all five lines plus the four extra points. To see this we have to check the lifting property: for fixed camera poses, adding the first four lines to their views uniquely recovers these four lines in space. For generic four lines $L_1, \ldots, L_4$ (shown in blue, green, orange and pink in Figure~\ref{fig:schubert}) in 3D, there are two lines (the black one and the grey one) which intersect $L_1, \ldots, L_4$.\footnote{Sottile, F.: Enumerative real algebraic geometry. In: Algorithmic and Quantitative Aspects of Real Algebraic Geometry in Mathematics and Computer Science. (2001)} Hence, every solution to the five-point problem can be lifted to a solution of the \PLkP{4} in Figure~\ref{fig:schubert} \emph{in two ways}.

By Lemma~\ref{lem:diagram-two-way}, we see that the \PLkP{4} is camera-minimal and has the same camera-degree as the five-point problem, namely $20$.
Moreover, the \PLkP{4} is balanced since its reduction to the balanced five-point problem removes $16$ degrees of freedom both in 3D and in the views.
Hence, the \PLkP{4} we constructed is indeed minimal.
Since each of the $20$ solutions of the five-points problem can be lifted to \emph{two} solutions of the \PLkP{4}, we have shown that its degree is $40$.
\end{example}

\noindent
\textbf{Proof of Theorems 1 and 4. }
Theorem 4 immediately follows from parts 1 and 2 of Lemma~\ref{lem:diagram-two-way}. To complete the proof of Theorem 1, suppose the upper \PLoneP{} $(p,l,\cI , \obs)$ is minimal. 
From parts 3 of lemmas~\ref{lem:diagram-one-way} and~\ref{lem:diagram-two-way}, we have that $\ppi$ and $\pi$ have equal fiber dimensions. 
It follows that the lower \PLoneP{} $(\PLPprime)$ is balanced, since
\begin{align*}
\dim (\cX_{p',l',\cI'} \times \cams{m}) &= \dim (\cX_{p,l,\cI}  \times \cams{m} ) -  \fiberPi \\
&= \dim (\YplIO ) - \fiberpi \\
&= \dim (\mathcal{Y}_{\PLPprime}).
\end{align*}
Moreover, from part 1 of Lemma~\ref{lem:diagram-one-way} we have that $\Phi '$ is dominant; so $(\PLPprime)$ is minimal. 

For generic image data $(x',\ell') \in \mathcal{Y}_{\PLPprime}$ and $(x,\ell ) \in \pi^{-1} (x', \ell ')$, it remains to show that the finite sets $\Phi^{-1} (x,\ell )$ and $(\Phi ')^{-1} (x',\ell ')$ have the same cardinality. 
By Theorem 4, we already know this for the camera solutions $\gamma \left( \Phi^{-1} (x,\ell ) \right) $  and $\gamma \left( (\Phi ')^{-1} (x',\ell ') \right).$ 
Since both problems are minimal, we are done by Lemma~\ref{lem:degs-equal}.
\qed

\subsection{Theorems~2, 3, 5, and 7}

In this subsection, we give full details on how the local features appearing in Theorems 2, 3, 5, and 7 are derived. 
The first Lemma~\ref{lem:invisible} explains why unobserved features cannot occur in neither minimal nor reduced problems.
The next two Lemmas~\ref{lem:pinSingleViewRed} and~\ref{lem:pinDoubleViewRed} give reduction rules for pins.

\begin{lemma}
\label{lem:invisible}
Consider a \PLoneP{} in three views.
If some point or line in space is not observed in any view, then the \PLoneP{} is reducible and not minimal.
\begin{itemize}
    \item For an unobserved line, the reduction forgets the line.
    \item For an unobserved point with at most one pin, the reduction forgets the point.
    \item For an unobserved point with at least two pins, the reduction forgets the point and all its pins.
\end{itemize}
\end{lemma}

\begin{proof}
For each bullet we check that the associated forgetting map $\Pi$ satisfies the conditions in the definition of reducibility. 
Feasibility holds vacuously for the first two bullets and is also easily seen for the third. 
Lifting solutions in the first two bullets is also trivial since $\Pi$ only forgets what is not observed at all.
This argument also shows that
$\fiberPi > 0 = \fiberpi$, 
so by part~3 of Lemma~\ref{lem:diagram-one-way}
the \PLoneP{} cannot be minimal in the first two bullets.

To verify the lifting property for the third bullet, we note that at most one of the pins can occur in each of the three views.  
By the first bullet, we may assume that all pins are visible. 
With these assumptions, it follows that the point has at most three pins. 
Hence, we are left with the following three cases:
\begin{center}
    \begin{tabular}{rccc}
    3-space & views & $\fiberPi$ & $\fiberpi$\\ \hline
    point with 2 pins  & 
    \begin{tabular}{c}
    \includegraphics[width=0.2\textwidth]{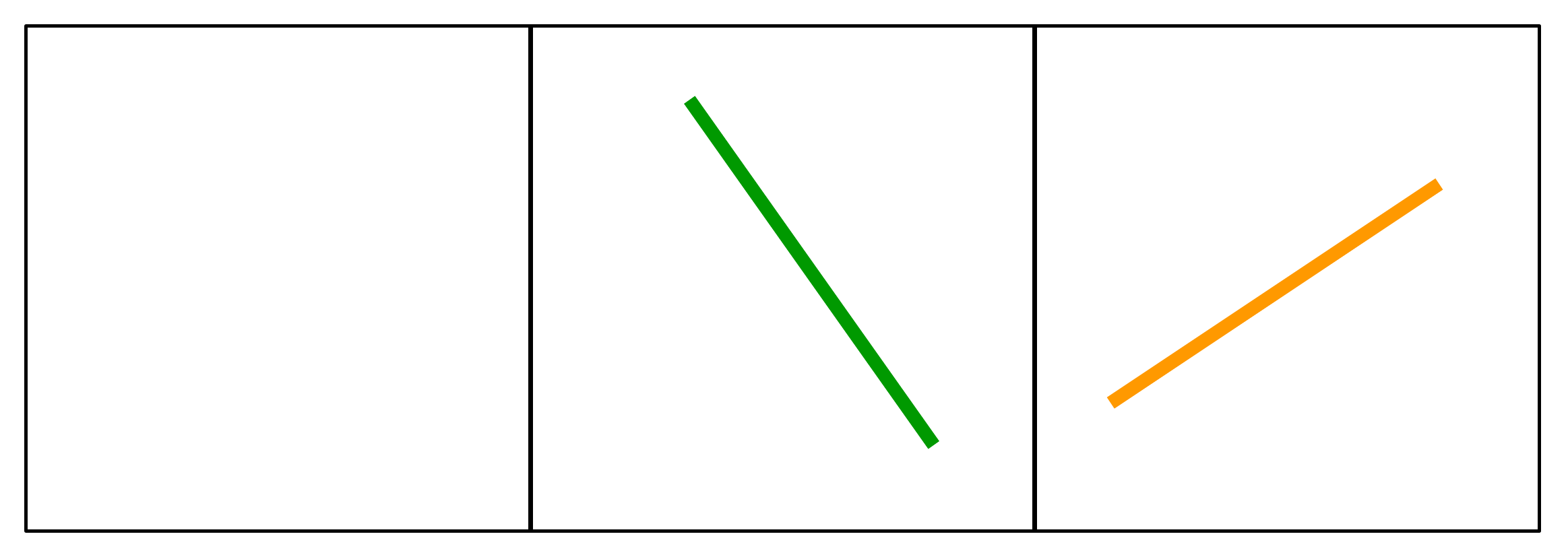}
    \end{tabular}
    & 7 & 4 \\
    point with 2 pins  &
    \begin{tabular}{c}
    \includegraphics[width=0.2\textwidth]{pix/pin2RedAll.pdf} 
    \end{tabular}
    & 7 & 6 \\
    point with 3 pins  & 
        \begin{tabular}{c}
    \includegraphics[width=0.2\textwidth]{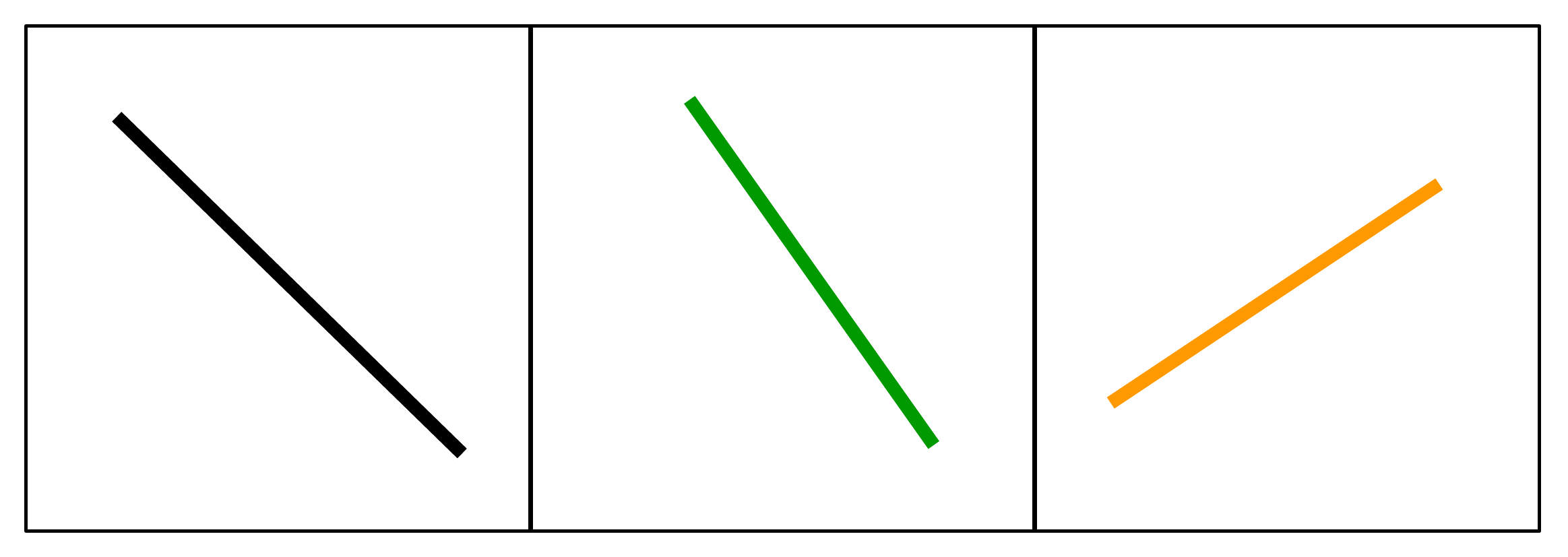}
    \end{tabular}
    & 9 & 6 
\end{tabular}
\end{center}
Now we can easily check that the lifting property is satisfied in each of these three cases.
For instance, for the last row (i.e. a point with three pins), 
the preimage of any three pins viewed like \includegraphics[width=0.09\textwidth]{pix/pin3.pdf} under fixed cameras is three planes in space which necessarily intersect in a point. 
Thus we can pick that point plus any three lines passing through that point and contained in the respective planes to lift solutions.
This shows that the \PLoneP{} is reducible as described in the third bullet.
Moreover, we see that $\fiberPi > \fiberpi$ holds in all three cases depicted above, so the \PLoneP{} cannot be minimal by part~3 of Lemma~\ref{lem:diagram-one-way}.
\qed
\end{proof}

\begin{lemma}
\label{lem:pinSingleViewRed}
Consider a \PLoneP{} in three views.
If some pin is observed in exactly one view such that the view also observes the point of the pin, 
then the \PLoneP{} is not minimal and it is reducible by forgetting the pin.
\end{lemma}

\begin{proof}
The forgetting map $\Pi$ which forgets the pin is clearly feasible. 
It also satisfies the lifting property:
for a fixed arrangement of cameras that view the point $X$ of the pin,
no matter how the pin is viewed in its single view, there is in fact a pencil of lines through $X$ yielding that view.
This shows that the \PLoneP{} is reducible.
Moreover, this reduction satisfies that
$$\fiberPi = 2 > 1 = \fiberpi;$$ 
so by part~3 of Lemma~\ref{lem:diagram-one-way}
the \PLoneP{} cannot be minimal.
\qed
\end{proof}

\begin{lemma}
\label{lem:pinDoubleViewRed}
Consider a  \PLoneP{} in three views.
If some pin is observed in exactly two views such that both views also observe the point of the pin, 
then the \PLoneP{} is reducible by forgetting the pin.
\end{lemma}

\begin{proof}
The forgetting map $\Pi$ which forgets the pin is clearly feasible.
Is also satisfies the lifting property: 
for a fixed arrangement of cameras that view the point $X$ of the pin, no matter how the pin is viewed in its two views, since it also passes through $X$ in both views, there is a unique line in 3D through $X$ yielding these two views.
This shows that the \PLoneP{} is reducible. 
\qed
\end{proof}


With the above lemmas in hand, we are able to enumerate a finite list of local features that may appear in a reduced (camera-)minimal \PLoneP{} in three views as well as a finite list of reduction rules to obtain such a \PLoneP{}.
To aid in this task, we list all possible ways in which free lines and points with $0,1$ or $2$ pins are viewed, and classify them according to whether or not a) they may appear in a minimal problem and b) they are reducible. 

Lemma~\ref{lem:tableComplete} below dispenses with cases involving local features with free lines and points with up to two pins that are already handled by Lemma~\ref{lem:invisible},~\ref{lem:pinSingleViewRed}, or~\ref{lem:pinDoubleViewRed}. 
The remaining cases are shown in Table~\ref{tab:proofAllTheorems}. 
For each of these local features, we may forget either a point and/or some number of lines.
Table~\ref{tab:proofAllTheorems} lists all feasible forgetting maps $\Pi$ for each observed local feature.
From this, we classify which observations of local features a) may appear in minimal problems and b) are reducible.

To determine reducibility of an observed local feature, we simply have to check if one of the listed feasible forgetting maps satisfies the lifting property.
Finding out if a local feature can be observed in a certain way in a minimal problem, is more subtle. 
We use the following two rules to exclude observed local features from appearing in minimal problems:
\begin{itemize}
    \item By Part~3 of Lemma~\ref{lem:diagram-one-way}, if $\fiberPi > \fiberpi$ for some feasible forgetting map $\Pi$, then the observed local feature cannot occur in any minimal \PLoneP{}. 
    \item A dangling pin (\ie a pin in 3D which is observed in a single view) cannot occur in any minimal \PLoneP{}.
\end{itemize}
All observations of local features which we cannot exclude from minimal problems using the two rules described above, we allow a priori to be part of minimal \PLoneP{}s. 
We mark this in Table~\ref{tab:proofAllTheorems} with a ``yes'' in the column ``minimal''\footnote{Our computations described in Section~\ref{sec:minimality} verify that actually all observed features marked as minimal in Table~\ref{tab:proofAllTheorems} do appear in some minimal \PLoneP{}s.}.


Finally, depending on the outcome of the reducibilty and minimality checks, we assign each observed local feature listed in Table~\ref{tab:proofAllTheorems} to one of the Theorems~2,3,5, or 7.
This assignment is summarized in Table~\ref{table:thm-classifier}.

\input{tables/thm-classifier}

\begin{lemma}
\label{lem:tableComplete}
All possibilities of how free lines and points with at most two pins are observed in three views are either treated by Lemmas~\ref{lem:invisible}, \ref{lem:pinSingleViewRed}, \ref{lem:pinDoubleViewRed} or appear in Table~\ref{tab:proofAllTheorems}.
\end{lemma}

\begin{proof}
By Lemma~\ref{lem:invisible}, we only have to record the cases where each point and line is observed at least once.
All such cases for free points and free lines are depicted in rows 1--6 of Table~\ref{tab:proofAllTheorems}.
So we are left to discuss points with one or two pins.

\textbf{Points with one pin:} We are distinguishing the different cases by how often the pin and its point are observed in the three views. We use the short notation $\boldsymbol{\lambda:\rho}$ to denote that the pin resp. its point is viewed $\lambda$ resp. $\rho$ times. By Lemma~\ref{lem:invisible}, we have that $\lambda, \rho \in \lbrace 1,2,3 \rbrace$.
\begin{itemize}
    \item[\textbf{1:1}\!\!]  By Lemma~\ref{lem:pinSingleViewRed}, we are left with the case where the pin and its point do not appear in the same view; see row~7 of Table~\ref{tab:proofAllTheorems}.
    \item[\textbf{1:2}\!\!] By Lemma~\ref{lem:pinSingleViewRed}, we can exclude the cases where one view sees the pin together with its point. Hence, we get that one view sees the pin and the other two views observe its point; see row~8 of Table~\ref{tab:proofAllTheorems}.
    \item[\textbf{1:3}\!\!] Here all three views observe the point and one of them also sees the pin. This is already handled by Lemma~\ref{lem:pinSingleViewRed} and thus does not appear in Table~\ref{tab:proofAllTheorems}.
    \item[\textbf{2:1}\!\!] Both such cases are shown in rows~9 and~10 of Table~\ref{tab:proofAllTheorems}.
    \item[\textbf{2:2}\!\!] By Lemma~\ref{lem:pinDoubleViewRed}, we may assume that the two views which see the pin do not both observe its point; see row~11  of Table~\ref{tab:proofAllTheorems}.
    \item[\textbf{2:3}\!\!] Here all three views observe the point and two of them also see the pin. This is already handled by Lemma~\ref{lem:pinDoubleViewRed} and thus does not appear in Table~\ref{tab:proofAllTheorems}.
    \item[\textbf{3:1}\!\!] -- \textbf{3:3} These cases are depicted in rows 12, 13 and 14 of Table~\ref{tab:proofAllTheorems}.
\end{itemize}

\textbf{Points with two pins:} We are distinguishing the different cases by how often each pin is observed in the three views. We use the short notation $\boldsymbol{\lambda_1:\lambda_2}$ to denote that the first resp. second pin is viewed $\lambda_1$ resp. $\lambda_2$ times. By Lemma~\ref{lem:invisible}, we have that $\lambda_1, \lambda_2 \in \lbrace 1,2,3 \rbrace$.
\begin{itemize}
    \item[\textbf{1:1}\!\!] By Lemma~\ref{lem:pinSingleViewRed}, we can exclude the cases where a view that sees one of the pins also observes the point. So we get that each view observing one of the pins does neither see the point nor the other pin; see row~15 of Table~\ref{tab:proofAllTheorems}.
    \item[\textbf{1:2}\!\!] By Lemma~\ref{lem:pinSingleViewRed}, we are left with the cases where the view observing the first pin does not see the point. So that view cannot see the other pin either.
    The other two views both observe the second pin, and at least one of them has to view the point.
    By Lemma~\ref{lem:pinDoubleViewRed}, we may assume that the point is not observed by both of these views;
    see row~16 of Table~\ref{tab:proofAllTheorems}.
    \item[\textbf{1:3}\!\!] Here all three views observe the second pin and one of them also sees the first pin, so also the point. This is already handled by Lemma~\ref{lem:pinSingleViewRed} and does not appear in  Table~\ref{tab:proofAllTheorems}.
    \item[\textbf{2:2}\!\!] By Lemma~\ref{lem:pinDoubleViewRed}, we can exclude the cases where two views observe both pins.
    Hence, we are left with the situation where one view sees both pins (and their point), and the other two views see one pin each. 
    Again by Lemma~\ref{lem:pinDoubleViewRed}, we may assume that none of the latter two views observes the point; see row~17 of Table~\ref{tab:proofAllTheorems}.
    \item[\textbf{2:3}\!\!] There are exactly two views where both pins, and hence the point are seen. We  apply Lemma~\ref{lem:pinDoubleViewRed} to see that this does not appear in Table~\ref{tab:proofAllTheorems}.
    \item[\textbf{3:3}\!\!] See row~18 of Table~\ref{tab:proofAllTheorems}. \qed
\end{itemize}
\end{proof}

Next, we prove Lemmas~\ref{lem:partial-pin-removal},~\ref{lem:hedgehog}, and~\ref{lem:complete-pin-removal}, which address the cases involving three or more pins. Along the way, we prove Lemma~\ref{lem:replacements}, which will also be useful in establishing Theorem~\ref{thm:camMinLiftToMin}.

\begin{lemma}
\label{lem:partial-pin-removal}
Consider a \PLoneP{} in three views.
If a point with at least three pins is not completely observed in the views (i.e. at least one view does not see at least one pin), 
then the \PLoneP{} is reducible by one of the cases in Lemmas~\ref{lem:invisible}, \ref{lem:pinSingleViewRed}, \ref{lem:pinDoubleViewRed}.
\end{lemma}

\begin{proof}
If the point $X$ or one of its pins is not observed in any view, then we are in the setting of Lemma~\ref{lem:invisible}.
Hence, we assume that the point $X$ and each of its pins is viewed at least once.
Let us first assume that the point $X$ is viewed by exactly one camera. 
The other two views can each see at most one of its pins.
Since $X$ has at least three pins in 3D, 
at least one of its pins must be only viewed by the same camera which sees the point $X$.
This situation is handled by Lemma~\ref{lem:pinSingleViewRed}.

It is left to consider the cases when $X$ is viewed by at least two cameras.
Let us assume that the point $X$ is observed in exactly two views.
As above, at least one of its pins is not observed by the view not seeing $X$.
So this pin has to be observed in either one or both of the views which observe the point $X$.
This shows that the \PLoneP{} is reducible by either
Lemma~\ref{lem:pinSingleViewRed} or Lemma~\ref{lem:pinDoubleViewRed}.

Finally, we assume that the point $X$ is observed in all three views.
Since the point $X$ and its pins are (by our assumption in Lemma~\ref{lem:partial-pin-removal}) not completely observed,
 one of its pins is seen in either one or  two views.
This is reducible by either
Lemma~\ref{lem:pinSingleViewRed} or Lemma~\ref{lem:pinDoubleViewRed}.
\qed
\end{proof}

Lemma~\ref{lem:hedgehog} shows that a point with $8$ pins cannot occur in any reduced camera-minimal problem; for minimal problems in complete visibility, this is already a result in~\cite{PLMP}. In the generality of camera-minimal problems, this result does not follow immediately from Lemma~\ref{lem:partial-pin-removal}. However, the next lemma lets us get around this, and also proves part of Theorem~\ref{thm:camMinLiftToMin}.

\begin{lemma}
\label{lem:replacements}
Applying any of the following replacements in images of a \PLoneP{} in three views preserves camera-minimality, camera-degrees, and reducedness:

\noindent
\PDoDt{.09} $\leftrightarrow$ \PoDoDt{.09}, \,
\PoDoDt{.09} $\leftrightarrow$ \PotDoDt{.09}, \,
\PPDo{.09} $\leftrightarrow$ \PoPDo{.09}
\end{lemma}

\begin{proof}
The three replacements have in common that they fixate a dangling pin (when read from left to right) or create a dangling pin (when read from right to left).
We prove the assertions in Lemma~\ref{lem:replacements} for all three replacement rules at once.
For this, we consider a \PLoneP{} $(\PLP)$ in three views that views one of its local features as depicted on the left of any of the three replacement rules.
After applying that replacement rule (from left to right) once, we obtain a new \PLoneP{} $(\PLPprime)$.
Clearly, every solution of $(\PLPprime)$ is also a solution of $(\PLP)$.
Moreover, for every solution $(X,P)$ of $(\PLP)$, there is in fact a one-dimensional set $\{(\tilde{X},P)\}$ of solutions of $(\PLP)$ with the same camera poses $P$,
where the 3D arrangements $\tilde{X}$ differ from the fixed 3D arrangement $X$ exactly by the dangling pin involved in the replacement rule.
Hence, one of these solutions $(\tilde{X},P)$ is also a solution of $(\PLPprime)$. This shows $$\cdeg(\PLP) = \cdeg(\PLPprime);$$
in particular, $(\PLP)$ is camera-minimal if and only if $(\PLPprime)$ is camera-minimal.

Furthermore, local features, which are observed in one of the five different ways present in the three replacement rules,
are not reducible; see also rows 8, 11, 15, 16, and 17 in Table~\ref{tab:proofAllTheorems}.
This means that the reducibility / reducedness of a point-line problem in three views does not depend on the appearance of local features observed as in the three replacement rules.
More precisely, $(\PLP)$ is reduced if and only if $(\PLPprime)$ is reduced.
\qed
\end{proof}

\begin{lemma}
\label{lem:hedgehog}
A reduced camera-minimal $\PLoneP{}$ in three views cannot have a point in 3D with eight pins.
\end{lemma}

\begin{proof}
We assume by contradiction that there is a reduced camera-minimal \PLoneP{} $(\PLP)$ in three views which has a point with eight pins in its 3D arrangement.
By Lemma~\ref{lem:partial-pin-removal}, this point and all its pins are observed in all views.
Since $(\PLP)$ is camera-minimal, its joint camera map is dominant.
In particular, we have that $\dim(\PplI \times \cams{3}) \geq \dim(\YplIO)$, i.e.
\begin{align}
    \label{eq:balancedInequality}
    11 \geq \dim(\YplIO) - \dim(\PplI).
\end{align}
If the 3D arrangement of  $(\PLP)$ consists only of the point with its eight pins, then~\eqref{eq:balancedInequality} is actually an equality, so $(\PLP)$ is a minimal \PLoneP{} completely observed by three calibrated views.
However, in~\cite{PLMP} it is shown that this point-line problem is \emph{not} minimal, a contradiction.

Hence, we see that the \PLoneP{} $(\PLP)$ has to contain at least one other local feature.
By Lemma~\ref{lem:tableComplete}, the only possible local features are either points with at least three pins or the local features listed in Table~\ref{tab:proofAllTheorems} which are not reducible (rows 3,5,6,8,9,11,13--18).
Since points with $k \geq 3$ pins have to be completely observed by Lemma~\ref{lem:partial-pin-removal}, 
they have more degrees of freedom in the 2D images ($=3(2+k)$) than in the 3D arrangement ($=3+2k$).
Similarly, all non-reducible features in Table~\ref{tab:proofAllTheorems}, except rows 15 and 16, 
have more degrees of freedom in the 2D images than in the 3D arrangement.
Since the inequality~\eqref{eq:balancedInequality} has to hold for the \PLoneP{} $(\PLP)$ and the point with eight pins already makes this inequality tight,
we have shown the following:
\begin{center}
    \begin{tabular}{cr}
    \emph{If the 3D arrangement of a reduced camera-minimal \PLoneP{} in three views} & $\hspace{4mm}(\star)$  \\
    \emph{contains a point with eight pins and at least one additional local feature,} & \\
    \emph{then it contains a point with two pins which is either viewed like} & \\
    \emph{\DoDtP{0.09} (row 15 in Table~\ref{tab:proofAllTheorems}) or \DoDtPt{0.09} (row 16 in Table~\ref{tab:proofAllTheorems}).} &
\end{tabular}
\end{center}

\smallskip
\noindent
In particular, the \PLoneP{} $(\PLP)$ has to contain a point with two pins viewed like \DoDtP{0.09} or \DoDtPt{0.09}.
We apply the replacements in Lemma~\ref{lem:replacements}
to obtain a new reduced camera-minimal \PLoneP{} in three views 
containing a point with eight pins and a point with two pins viewed like \PotDoDt{.09},
such that none of its points with two pins is observed like \DoDtP{0.09} nor \DoDtPt{0.09}.
This contradicts $(\star)$.
\qed
\end{proof}

Finally, we combine everything we have learned so far about pins to bound the maximum number of pins per point and the maximum number of points with many pins in reduced camera-minimal \PLoneP{}s in three views.
Afterwards we are ready to summarize our findings to provide proofs for Theorems 2, 3, 5, and 7.

\begin{lemma}
\label{lem:complete-pin-removal}
A reduced camera-minimal \PLoneP{} in three views has at most one point with three or more pins. If such a point exists, 
\begin{itemize}
    \item it has at most seven pins,
    \item and the point and all its pins are observed in all three views.
\end{itemize}
\end{lemma}

\begin{proof}
By Lemma~\ref{lem:partial-pin-removal}, every point with three or more pins  has to be completely observed in a reduced \PLoneP{} in three views.
Hence, it is left to show that 1) at most one such point with many pins exists, and that 2) it has at most seven pins.

For the first assertion, we denote by $\rho$ the number of points in 3D which have three or more pins. 
We consider the forgetting map $\Pi$ which forgets everything, except these $\rho$ points with exactly three of their pins each.
We obtain a diagram as in Figure~\ref{fig:commutativeDiagram}.
Since this forgetting map $\Pi$ is feasible and the given (upper) \PLoneP{} is camera-minimal,
Lemma~\ref{lem:diagram-one-way} implies that the joint camera map of the resulting (lower) \PLoneP{} is dominant. 
In particular, the resulting \PLoneP{} satisfies
\begin{align*}
    9\rho + 11 = \dim(\mathcal{X}_{p',l',\cI'} \times \cams{3}) \geq \dim(\mathcal{Y}_{\PLPprime}) = 15\rho,
\end{align*}
i.e. $11 \geq 6 \rho$. Thus, we see that $\rho \leq 1$, which means that the given (upper) \PLoneP{} has at most one point in 3D with three or more pins.

Finally, we show that such a point has at most seven pins, if it exists.
We assume $\rho=1$ and denote by $\lambda \geq 3$ the number of pins at that point. 
We consider the forgetting map $\Pi$ which forgets everything, except that single point with its $\lambda$ pins.
As before, we see that the joint camera map of the resulting lower \PLoneP{} is dominant, which yields that
\begin{align*}
    3 + 2\lambda + 11 = \dim(\mathcal{X}_{p',l',\cI'} \times \cams{3}) \geq \dim(\mathcal{Y}_{\PLPprime}) = 3(2+\lambda),
\end{align*}
so $8 \geq \lambda$. 
By Lemma~\ref{lem:hedgehog}, we have that $\lambda \leq 7$, which concludes the proof.
\qed
\end{proof}


\noindent
\textbf{Proof of Theorems~\ref{thm:uniqueReduced} and~\ref{thm:cameraMinimalReduction}.}
We first show that each \PLoneP{} in three views is reducible to a unique reduced  \PLoneP{}.
For this, we notice that a \PLoneP{} is reducible if and only if one of the local features in its 3D arrangement is reducible.
Hence, we only have to check that each possible local feature is reducible to a unique reduced local feature. 
By Lemmas~\ref{lem:tableComplete} and~\ref{lem:partial-pin-removal},  this assertion follows from examining the reducible cases in Table~\ref{tab:proofAllTheorems} together with Lemmas~\ref{lem:invisible},\ref{lem:pinSingleViewRed},\ref{lem:pinDoubleViewRed}.

Moreover, we obtain all reduction rules listed in Theorem~\ref{thm:cameraMinimalReduction} by collecting the forgetting maps described in Lemmas~\ref{lem:invisible},\ref{lem:pinSingleViewRed},\ref{lem:pinDoubleViewRed}, as well as the reducible observed features in Table~\ref{tab:proofAllTheorems}.
Among those, Lemma~\ref{lem:pinDoubleViewRed} as well as rows 2, 10, and 12 of Table~\ref{tab:proofAllTheorems} are applicable for minimal problems; 
thus these reduction rules are listed in Theorem~\ref{thm:uniqueReduced}.

Finally, we address the last assertion in Theorem~\ref{thm:uniqueReduced}. Considering a commutative diagram as in Figure~\ref{fig:commutativeDiagram} which is obtained from one of the four reduction rules listed in Theorem~\ref{thm:uniqueReduced},
we assume that the lower \PLoneP{} is minimal and aim to prove that the upper \PLoneP{} is minimal as well.
Since the lower \PLoneP{} is minimal, it is balanced and its joint camera map is dominant.
By Lemma~\ref{lem:diagram-two-way},
the joint camera map of the upper \PLoneP{} is also dominant.
Furthermore, the four forgetting maps $\Pi$ listed in Theorem~\ref{thm:uniqueReduced} satisfy that 
$\fiberPi = \fiberpi$.
Since the lower \PLoneP{} is balanced, we see from Lemma~\ref{lem:forget-irreducible} that
\begin{align*}
    \dim (\PplI \times \cams{3}) &= \dim(\mathcal{X}_{p',l',\cI'} \times \cams{3}) + \fiberPi \\
    &=\dim(\mathcal{Y}_{\PLPprime}) + \fiberpi 
    =\dim(\YplIO).
\end{align*}
So the upper \PLoneP{} is also balanced, hence minimal.
\qed

\bigskip
\noindent
\textbf{Proof of Theorems~\ref{thm:reducedMinimalLooks} and~\ref{thm:danglingPins}.}
The parts of Theorem~\ref{thm:reducedMinimalLooks} (and Theorem~\ref{thm:danglingPins}) addressing points with three or more pins have already been proven in Lemma~\ref{lem:complete-pin-removal}.
To find all ways of how free lines and points with at most two pins can be observed by reduced camera-minimal \PLoneP{}s in three views, 
it is enough (by Lemma~\ref{lem:tableComplete}) 
to gather the non-reducible observed features in Table~\ref{tab:proofAllTheorems}.
Among those, the ones marked as minimal are listed in Theorem~\ref{thm:reducedMinimalLooks}, the others in Theorem~\ref{thm:danglingPins}.
\qed

\subsection{\Cref{thm:camMinLiftToMin} and its corollaries}


\noindent
\textbf{Proof of \Cref{thm:camMinLiftToMin}.}
The nature of the four replacements is to introduce another image of the dangling pin into one of the views. This makes this pin  reconstructable in 3D in the resulting \PLoneP{}, but produces no additional constraint on the cameras. 

Formally, since the initial problem is reduced and camera-minimal, Lemma~\ref{lem:replacements} tells us that the resulting \PLoneP{} is also reduced and camera-minimal, with the same camera-degree.
It is left to show that the resulting \PLoneP{} is indeed minimal.
By Theorem~\ref{thm:danglingPins}, we see that the resulting \PLoneP{} only has local features as listed in Theorem~\ref{thm:reducedMinimalLooks}, 
i.e. points with three or more pins completely observed in all views plus some of the features shown in Table~\ref{tab:localFeatures}.
All of these features can be uniquely recovered in 3D  for a fixed camera solution.    
\qed

\smallskip
\noindent
\textbf{Proof of \Cref{cor:swap}.}
By Lemma~\ref{lem:replacements}, replacing a single occurrence of \PoPDo{.09} in a minimal problem with \PPDo{.09}, yields a camera-minimal problem with the same camera-degree. 
Now we can replace \PPDo{.09}  with \PPoDo{.09} to obtain another camera-minimal \PLoneP{} of the same camera-degree.
As argued above, this resulting \PLoneP{} is actually minimal. 
By Lemma~\ref{lem:degs-equal}, it has the same degree as the initial minimal problem.
\qed

\smallskip
\noindent
\textbf{Proof of \Cref{cor:swap-label-equivalence}.}
Starting from a terminal camera-minimal \PLoneP{} in three views, the lift in \Cref{thm:camMinLiftToMin} produces possibly several reduced minimal \PLoneP{}s: these only differ depending on whether the third or the fourth replacement in \Cref{thm:camMinLiftToMin} is applied to each point with one dangling pin viewed like \PPDo{.09} (up to relabeling the views).
So all of the resulting minimal problems lie in the same swap\&label-equivalence class.

Starting from a reduced minimal \PLoneP{} in three views, we apply the  replacements in \Cref{lem:replacements}, read from right to left,  until
\PoDoDt{.09}, \PotDoDt{.09}, and \PoPDo{.09}
do not appear any longer.
This, by \Cref{def:terminal}, results in a terminal problem.
Moreover, all reduced minimal \PLoneP{}s which are related via the swaps in \Cref{cor:swap} yield the same resulting terminal problem.

This explains the one-to-one correspondence in \Cref{cor:swap-label-equivalence}.
It preserves camera-degrees by \Cref{lem:replacements}.
\qed

\section{Computations}

\subsection{Swap\&label-equivalence classes of signatures}

\label{subsec:swapLabel}

As noted in  Section~\ref{sec:balancedPL1Ps}, each reduced minimal \PLoneP{} in three views can be encoded as a \emph{signature}, which is an integer solution to the dimension-count equation~\eqref{eq:balanced}.
Thus,  a signature is simply an integer vector of length $27.$  In general, we can enumerate all nonnegative integer solutions $(k_1, \ldots , k_l)$ to $a_1 k_1+\cdots + a_l k_l =n$ by recursively solving $a_2 k_2 + \ldots + a_l k_l = n - j a_1 $ for $j= \lfloor n / a_1 \rfloor , \ldots , 0.$  The result of this enumeration procedure is a list of solutions that is sorted decreasingly with respect to the usual lexicographic order on $\ZZ^{27}$.
This is how we computed all $845161$ solutions of~\eqref{eq:balanced}.
\\\\
For each signature, we find all \PLoneP{}s which are the same up to relabeling of the views---that is, we mod out the action of the permutation group $S_3$ that permutes the three views. Note that if we take one representative of each orbit of this action in the order they appear in the solution list from above, this ensures that the lex order from this step is preserved.\\\\
The procedure above gives us  143494 label-equivalence classes of reduced \PLoneP{}s. It remains to extract a representative of each swap\&label-equivalence class from this list. Recall that a \emph{swap} operation exchanges orderered local features of the form \PoPDo{.09} and \PPoDo{.09}, and that two \PLoneP{}s are \emph{swap\&label-equivalent} if they differ only by some sequence of swaps and $S_3$-permutations of the views. In our implementation, the coordinates of the signature vectors which participate in swaps are indexed from $12$ to $17$ inclusive.  They are arranged such that $12$ resp. $13$ index the ordered local features \PPoDo{.09} resp. \PoPDo{.09} ---similarly for $14,15$ and $16,17.$ Thus, we may restrict attention to those signatures whose coordinates $13,15,$ and $17$ equal zero; if any of these coordinates is nonzero, then we may find a swap\&label-equivalent representative earlier in the list. To see this, note that if we swap $13$ to $12,$ $15$ to $14,$ and $17$ to $16,$ we get a lexicographically larger signature; moreover, the signature is also lex-maximal in its label-equivalence class, and hence occurs in the list of 143494 representatives.\\\\
The previous paragraph identifies $\approx 40,000$ swap\&label-equivalent pairs, but does not yet yield a unique representative for each class. To do this, we iterate over the list of remaining signatures in order, maintaining a single representative per swap\&label-equivalence class encountered so far. For each signature, we enumerate its $S_3$-orbit and perform the swaps $13$ to $12,$ $15$ to $14,$ and $17$ to $16$ to each element in this orbit. This operation produces $5$ additional signatures, and we must delete any that appear later in the list. In the end, we are left with 76446 signatures. 

\subsection{Checking minimality}
As mentioned in Section~\ref{sec:minimality}, our rank check over the finite field $\mathbb{F}_q$ may be susceptible to \emph{false negatives}---in other words, it is possible that we may incorrectly conclude that a problem is not minimal due to unlucky random choices made during the computation. On the other hand, \emph{false positives} are impossible---we now explain this in detail.
Let $(\PLP)$ be any balanced point-line problem. If we parametrize cameras by a rational map $P: \CC^{11} \to \cams{3} ,$ then the exceptional set of complex $(X,t)$ such that the Jacobian of $\Phi_{\PLP}$ at $(X,P(t))$ drops rank is Zariski-closed. It follows that there is a point $(X_0,t_0)$ with integer coordinates outside of this exceptional set. Passing to residues modulo some prime $q,$ the rank of the Jacobian at this point can only drop. Thus if we find a point with av Jacobian of full rank modulo $q$, then we may conclude that the joint camera map $\Phi_{\PLP}$ is dominant, i.e. that the balanced $(\PLP )$ is minimal. 

\subsection{Computing degrees}

The main results of Section~\ref{sec:degrees}, namely Results~\ref{result:smalldegs},~\ref{res:p1l1p} and~\ref{res:pl0p}, are based on our computation of degrees for minimal and camera-minimal problems using monodromy. In this context, the term ``monodromy" refers to a paradigm of numerical continuation methods for solving parametrized polynomial systems that collect solutions for random parameter values in a one-by-one manner.  We refer to~\cite{Duff-Monodromy} for a detailed description of the ideas involved and discussion of implementation issues. Most relevant for our purposes is that this computation can be aborted early; ignoring numerical subtleties, we can then say that the number of solutions obtained at this point is a lower bound on the degree.

Our implementation of degree computation has the option of using one of two formulations. In the first, we solve explicitly for world features as well as camera matrices. In the second, \emph{eliminated formulation}, we only solve for camera matrices using determinantal constraints as in~\cite[Sec 6]{PLMP}. Our computational results use the eliminated formulation for efficiency. As a sanity check for some problems of interest, we ran monodromy until it stabilized for both world and eliminated formulations to confirm the conclusion of Lemma 1. 

Since our degree computations are randomized and susceptible to the possible failures of numerical continuation methods, our operational definition of ``success" requires that all algebraic constraints are satisfied by the solutions collected and that either some target degree was exceeded (eg.~300 in Result~\ref{result:smalldegs}) or that monodromy has \emph{stabilized} in the sense that the random problem instances have collected the same number of solutions with no further progress after several iterations. We note that more sophisticated stopping criteria are available for monodromy~\cite{LRS18,HR19} but are generally much more expensive. For problems of interest (eg.~those appearing in \cite{Kileel-MPCTV-2016,PLMP} and~\Cref{tab:PL0P}), the stabilization heuristic allowed us to recover all previously known degrees, though some problems needed to be run more than once due to numerical failures. We also reran several cases appearing in Table~\ref{tab:smalldegs} according to this criteria in order to gain more confidence in the reported degree.

\subsection{Finding subfamilies}

As noted in Section~\ref{sec:degrees}, there are several subfamilies of minimal and camera-minimal \PLoneP{}s that are of interest: namely, the \PLzeroP{}s occuring in Result~\ref{res:pl0p}, the \PLoneP{}s with at most one pin per point in Result~\ref{res:p1l1p}, and the extensions of the five-point problem from Result~\ref{res:registration}.
We point out that the signature vectors described in Section~\ref{subsec:swapLabel} give a complete combinatorial description of each problem; in particular, in which views certain points and lines are seen and which incidences they have. Thus, it is straightforward to enumerate these subfamilies starting from the lists of reduced minimal and terminal camera-minimal problems.


\begin{table}
    \centering
    \resizebox{0.86\textwidth}{!}{
\input{tables/allTheoremTable.tex}
    
    }
    \caption{  All ways to observe free lines and points with $\leq 2$ pins in $3$ views (modulo cases treated by Lemmas~\ref{lem:invisible},\ref{lem:pinSingleViewRed},\ref{lem:pinDoubleViewRed})  and their appearance in Theorems $2,3,5$ or $7$.
    For each observed feature,  all possible feasible forgetting maps $\Pi$ are listed, distinguished by the column ``forget'': ``P'' resp. ``L'' denotes a forgotten point resp. line.
    }
    \label{tab:proofAllTheorems}
\end{table}

\newpage

\section{Glossary of assumptions, properties and concepts}
\label{sec:glossary}

Here we provide a glossary of assumptions, properties and concepts used in the paper. Our aim is to give a concise and intuitive exposition of concepts used. 
\begin{enumerate}
\item Realizability of $\mathcal{I}$ -- the incidence relations are realizable by some point-line arrangement in $\RR^3$.
\item Completeness of $\mathcal{I}$ -- every incidence which is automatically implied by the incidences in $\cI$ must also be contained in $\cI$.
\item Completeness of  $\mathcal{O}$ -- if a camera observes two lines that meet according to $\mathcal{I}$, then it observes their point of intersection.
\item Genericity of $\cams{m}$ -- the points and lines in the views are in generic positions with respect to the specified incidences $\mathcal{I}$, \ie random noise in image measurements does not change the number of solutions when the incidences from 3D are not broken by noise in images. 
\item Complete visibility --  all points and lines are observed in all images and all observed information is used to formulate minimal problems.
\item Partial visibility --  some points and lines may be forgotten when projecting into images.
\item \PLkP{k} -- each line in 3D is incident to at most $k$ points.
\item Dominant -- ``almost all'' 2D images have a solution (i.e. a 3D arrangement and cameras yielding the images)
\item Balanced -- the number of DOF (\ie the dimensions of varieties) describing cameras and preimages in 3D are equal to the number of DOF describing image measurements. 
\item Minimal -- balanced and dominant. This is equivalent to that ``almost all'' images in 2D have a positive finite number of solutions.
\item Camera-Minimal -- the number of solutions in camera parameters is finite and positive (despite that infinitely many solutions may exist for 3D structures). 
\item Reduced problem -- a problem where all viewed features imply nontrivial constraints on the cameras
\item Degree of a minimal problem -- the number of point-line configurations and camera poses consistent with generic image data
\item Camera-degree of a camera-minimal problem -- the number of camera poses consistent with generic image data
\end{enumerate}

\ifarxiv
\noindent
{\small \textbf{Acknowledgements.} We are grateful to ICERM (NSF DMS-1439786 and the Simons Foundation grant 507536) for the hospitality from September 2018 to February 2019, where this project has started. We also thank the many research visitors at ICERM who participated in fruitful discussions on minimal problems. Research of T. Duff and A. Leykin is supported in part by NSF DMS-1719968. T. Duff also acknowledges support from a fellowship from the Algorithms and Randomness Center at Georgia Tech and the Max Planck Institute for Mathematics in the Sciences in Leipzig. 
K. Kohn was partially supported by the Knut and Alice Wallenberg Foundation within their WASP (Wallenberg AI, Autonomous Systems and Software Program) AI/Math initiative.
T. Pajdla was supported by the European Regional Development Fund under the project IMPACT (reg. no. CZ.02.1.01/0.0/0.0/15 003/0000468) and by EU H2020 No.~856994 ARtwin Project.}
\fi

\end{document}

%% file: definitions.tex
\usepackage{graphicx}
\usepackage{amssymb}
\usepackage{amsmath,graphicx}
\usepackage{adjustbox}
\usepackage{tikz}
\usetikzlibrary{cd}
\usetikzlibrary{backgrounds}
\ifarxiv\else
\usepackage{ruler}
\fi
\usepackage{color}
\usepackage[width=122mm,left=12mm,paperwidth=146mm,height=193mm,top=12mm,paperheight=217mm]{geometry}
\usepackage{array}
\usepackage{multirow}
\usepackage[english]{babel}
\usepackage[utf8]{inputenc}
\usepackage{hyperref}
\usepackage{makecell}
\usepackage{blkarray}
\usepackage{arydshln}
\usepackage{comment}
\usepackage[capitalise]{cleveref}
\allowdisplaybreaks
\usepackage{array}
\newcounter{rowcount}
\newtheorem{result}[theorem]{Result}

\newcommand{\ie}{i.e.~}
\newcommand{\eg}{e.g.~}
\newcommand{\csev}{c_7}
\newcommand{\csix}{c_6}
\newcommand{\cfiv}{c_5}
\newcommand{\cfou}{c_4}
\newcommand{\cthr}{c_3}
\newcommand{\ctwofull}{c_{2,0}}
\newcommand{\ctwoparA}{c_{2,1}^{a}}
\newcommand{\ctwoparB}{c_{2,1}^{b}}
\newcommand{\ctwoparC}{c_{2,1}^{c}}
\newcommand{\conefull}{c_{1,0}}
\newcommand{\coneparoneA}{c_{1,1}^a}
\newcommand{\coneparoneB}{c_{1,1}^b}
\newcommand{\coneparoneC}{c_{1,1}^c}
\newcommand{\conepartwoA}{c_{1,2}^a}

\newcommand{\conepartwoF}{c_{1,2}^f}
\newcommand{\coneparthrA}{c_{1,3}^a}
\newcommand{\coneparthrB}{c_{1,3}^b}
\newcommand{\coneparthrC}{c_{1,3}^c}
\newcommand{\czerfull}{c_{0,0}}
\newcommand{\czerparA}{c_{0,1}^a}
\newcommand{\czerparB}{c_{0,1}^b}
\newcommand{\czerparC}{c_{0,1}^c}
\newcommand{\cfree}{f}
\definecolor{forest}{RGB}{11,128,35}

\newcommand{\cX}{{\mathcal X}}
\newcommand{\cI}{{\mathcal I}}

\newcommand{\cams}[1]{\mathcal{C}_{#1}}

\newcommand{\obs}{\mathcal{O}}
\newcommand{\PLP}{p, l, \mathcal{I}, \obs}
\newcommand{\PLPprime}{p', l', \mathcal{I}', \obs'}
\newcommand{\cdeg}{\mathrm{cdeg}}
\newcommand{\PLkP}[1]{\text{PL${}_{#1}$P}}
\newcommand{\PLzeroP}{\PLkP{0}}
\newcommand{\PLoneP}{\PLkP{1}}

\newcommand{\PplI}{\cX_{p,l,\cI}}
\newcommand{\YplIO}{\mathcal{Y}_{p,l,\cI,\obs}}
\newcommand{\ppi}{\mathrm{\Pi}}

\newcommand{\fiberPi}{\dim(\mathrm{fiber}(\Pi))}
\newcommand{\fiberpi}{\dim(\mathrm{fiber}(\pi))}
\newcommand{\SO}{\mathrm{SO}}
\newcommand{\ZZ}{\mathbb{Z}}
\newcommand{\CC}{\mathbb{C}}

\newcommand{\PP}{\mathbb{P}}
\newcommand{\RR}{\mathbb{R}}
\newcommand{\GG}{\mathbb{G}}
\newcommand{\FF}{\mathbb{F}}

\newcommand{\yesmark}{\checkmark}
\newcommand{\nomark}{\text{\sffamily X}}

\usepackage{scalerel}
\newlength\bshft
\def\fakebold#1{\ThisStyle{\ooalign{$\SavedStyle#1$\cr%
  \kern-\bshft$\SavedStyle#1$\cr%
  \kern\bshft$\SavedStyle#1$}}}

\newlength{\singleheight}
\newlength{\singlewidth}
\newcommand{\FFF}[1]{\includegraphics[width=#1\textwidth]
{pix/freeline.pdf}}
\newcommand{\PDoDt}[1]{\includegraphics[width=#1\textwidth]
{pix/dangling22.pdf}}
\newcommand{\PotDoDt}[1]{\includegraphics[width=#1\textwidth]
{pix/pin2param4.pdf}}
\newcommand{\PoDoDt}[1]{\includegraphics[width=#1\textwidth]
{pix/dangling21.pdf}}
\newcommand{\PPDo}[1]{\includegraphics[width=#1\textwidth]
{pix/dangling11.pdf}}
\newcommand{\PoPDo}[1]{\includegraphics[width=#1\textwidth]
{pix/pin1param2.pdf}}
\newcommand{\PPoDo}[1]{\includegraphics[width=#1\textwidth]
{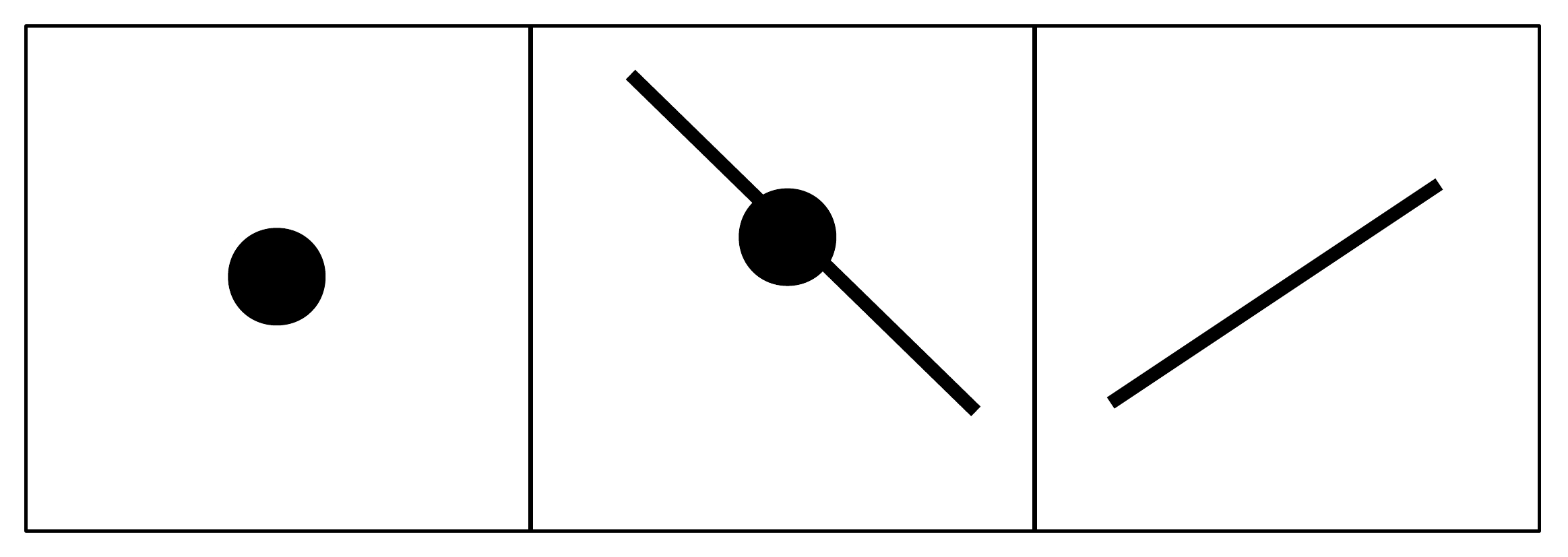}}
\newcommand{\FFN}[1]{\includegraphics[width=#1\textwidth]
{pix/freelineRed.pdf}}
\newcommand{\FNN}[1]{\includegraphics[width=#1\textwidth]
{pix/freelineSIngleRed.pdf}}
\newcommand{\PNN}[1]{\includegraphics[width=#1\textwidth]
{pix/point.pdf}}
\newcommand{\PPP}[1]{\includegraphics[width=#1\textwidth]
{pix/points3Red.pdf}}
\newcommand{\PPN}[1]{\includegraphics[width=#1\textwidth]
{pix/pin0param2.pdf}}
\newcommand{\FPN}[1]{\includegraphics[width=#1\textwidth]
{pix/pin1Red.pdf}}
\newcommand{\PoFN}[1]{\includegraphics[width=#1\textwidth]
{pix/pin1RedAll.pdf}}
\newcommand{\PPF}[1]{\includegraphics[width=#1\textwidth]
{pix/dangling11.pdf}}
\newcommand{\FPF}[1]{\includegraphics[width=#1\textwidth]
{pix/pin1param3.pdf}}
\newcommand{\PPoF}[1]{\includegraphics[width=#1\textwidth]
{pix/pin1param2v2.pdf}}
\newcommand{\PoFF}[1]{\includegraphics[width=#1\textwidth]
{pix/pin1RedPoint.pdf}}
\newcommand{\PoPoF}[1]{\includegraphics[width=#1\textwidth]
{pix/pin1param1.pdf}}
\newcommand{\PoPoPo}[1]{\includegraphics[width=#1\textwidth]
{pix/pin1param0.pdf}}
\newcommand{\DoDtP}[1]{\includegraphics[width=#1\textwidth]
{pix/dangling22.pdf}}
\newcommand{\DoDtPt}[1]{\includegraphics[width=#1\textwidth]
{pix/dangling21.pdf}}
\newcommand{\DoPotDt}[1]{\includegraphics[width=#1\textwidth]
{pix/pin2permC.pdf}}
\newcommand{\PotPotPot}[1]{\includegraphics[width=#1\textwidth]
{pix/pin2param0.pdf}}

\newcommand{\orange}[1]{\textcolor{orange}{#1}}
\definecolor{mygreen}{HTML}{088527}
\newcommand{\green}[1]{\textcolor{mygreen}{#1}}

\newcommand{\tw}{0.075}

%% file: tables/thm-classifier.tex
\begin{table}[h!]
    \centering
  \begin{tabular}{c|rm{5.5em}l}
  & &  
  \begin{tabular}{c}
  can occur in\\
  a minimal\\
  problem? 
  \end{tabular}
  &  \\
  & YES & &  NO\\
    \hline 
  reduced 
  & 
  \begin{tabular}{r}
  Theorem 3\\
 \FFF{0.1}
  \end{tabular}
  & & 
    \begin{tabular}{l}
  Theorem 5\\
 \PPDo{0.1}
  \end{tabular}
  \\[1em]
  reducible
  &
    \begin{tabular}{r}
  Theorem 2\\
 \FFN{0.1}
  \end{tabular}
  & & 
\begin{tabular}{l}
  Theorem 7\\
 \FNN{0.1}
  \end{tabular}
  \end{tabular}
    \caption{Local features observed in three views pertaining to Theorems 2, 3, 5, and 7 with examples.}
    \label{table:thm-classifier}
\end{table}

%% file: tables/allTheoremTable.tex
    \begin{tabular}{ccc|ccccccc}
   3D && 2D &  & dim. of & dim. of & lifting  &  &  & \vspace{-1mm} \\ 
    feature && views & forget & $\mathrm{fiber}(\Pi)$ & $\mathrm{fiber}(\pi)$ &  property & reducible & minimal & Thm.\\
    \hline 
    \multirow{3}{*}{free line}
    &1) & \FNN{0.075}     & L & 4 & 2 & yes & yes & no & 7 \\
    &2) &\FFN{0.075}     & L & 4 & 4 & yes & yes & yes & 2 \\ 
    &3) & \FFF{0.075} & L & 4 & 6 & no & no & yes & 3 \\
    \hline 
    \multirow{3}{*}{point}
    &4) & \PNN{\tw} & P & 3 & 2 & yes & yes & no & 7\\
    &5) & \PPN{\tw} & P & 3 & 4 & no & no & yes & 3\\
    &6) & \PPP{\tw} & P & 3 & 6 & no & no & yes & 3\\
    \hline
    \multirow{24}{*}{$\begin{array}{c}
         \text{point} + \\ 1 \text{ pin} 
    \end{array}$}
    &\multirow{3}{*}{7)} & \multirow{3}{*}{\FPN{\tw}}
    & L & 2 & 2 & no\\
    &&& P & 1 & 2 & no & yes & no & 7\\
    &&& P+L & 5 & 4 & yes \\
    \cdashline{2-10}
    &\multirow{3}{*}{8)} & \multirow{3}{*}{\PPF{\tw}}
    & L & 2 & 2 & no \\
    &&& P & 1 & 4 & no & no & no & 5\\
    &&& P+L & 5 & 6 & no \\
    \cdashline{2-10} 
    &\multirow{3}{*}{9)} & \multirow{3}{*}{\FPF{\tw}} 
    & L & 2 & 4 & no \\
    &&& P & 1 & 2 &  no & no & yes & 3\\
    &&& P+L & 5 & 6 & no\\
    \cdashline{2-10}
    &\multirow{3}{*}{10)} & \multirow{3}{*}{\PoFN{\tw}}
    & L & 2 & 3 & no \\
    &&& P & 1 & 1 & yes & yes & yes & 2 \\
    &&& P+L &  5 & 5 & yes \\
    \cdashline{2-10}
    &\multirow{3}{*}{11)} & \multirow{3}{*}{\PPoF{\tw}}
    & L & 2 & 3 & no \\
    &&& P & 1 & 3 & no & no & yes & 3 \\
    &&& P+L &  5 & 7 & no \\
    \cdashline{2-10}
    &\multirow{3}{*}{12)} & \multirow{3}{*}{\PoFF{\tw}}
    & L & 2 & 5 & no\\
    &&& P & 1 & 1 & yes & yes & yes & 2 \\
    &&& P+L & 5 & 7 & no\\
    \cdashline{2-10}
    &\multirow{3}{*}{13)} & \multirow{3}{*}{\PoPoF{\tw}} 
    & L & 2 & 4 & no \\
    &&& P & 1 & 2 & no & no & yes & 3\\
    &&& P+L & 5 & 8 & no \\
    \cdashline{2-10} 
    &\multirow{3}{*}{14)} & \multirow{3}{*}{\PoPoPo{\tw}} 
    & L & 2 & 3 & no \\
    &&& P & 1 & 3 & no & no & yes & 3\\
    &&& P+L & 5 & 9 & no \\
    \hline
    \multirow{18}{*}{$\begin{array}{c}
         \text{point} + \\ 2 \text{ pins} 
    \end{array}$}
    &\multirow{4}{*}{15)} & \multirow{4}{*}{\DoDtP{\tw}}
    & \orange{L} & 2 & 2 & no\\
    &&& \orange{L}+\green{L} & 4 & 4 & no & no & no & 5 \\
    &&& P+\orange{L} & 3 & 4 & no \\
    &&& P+\orange{L}+\green{L} & 7 & 6 & no \\
    \cdashline{2-10}
 &\multirow{6}{*}{16)} & \multirow{6}{*}{\DoDtPt{\tw}} 
    & \orange{L} & 2 & 2 & no \\
    &&& \green{L} & 2 & 3 & no\\
    &&& \orange{L}+\green{L} & 4 & 5 & no & no & no & 5  \\
    &&& P+\orange{L} & 3 & 3 & no \\
    &&& P+\green{L} & 3 & 5 & no \\
    &&& P+\orange{L}+\green{L} & 7 & 7 & no\\
    \cdashline{2-10}
    &\multirow{4}{*}{17)} & \multirow{4}{*}{\DoPotDt{\tw}}
 & \orange{L} & 2 &3 & no\\
 &&& \orange{L}+\green{L} & 4 & 6 & no & no & yes & 3 \\
 &&& P+\orange{L} & 3 & 4 & no\\
 &&& P+\orange{L}+\green{L} & 7 & 8 & no \\
 \cdashline{2-10}
 &\multirow{4}{*}{18)} & \multirow{4}{*}{\PotPotPot{\tw}} 
  & \orange{L} & 2 & 3 & no \\
 &&& \orange{L}+\green{L} & 4 & 6 & no & no & yes & 3\\
 &&& P+\orange{L} & 3 & 6 & no\\
 &&& P+\orange{L}+\green{L} & 7 & 12 & no 
    \end{tabular}